%% file: main.tex
\documentclass{article}

\usepackage{fullpage}

\usepackage[T1]{fontenc}

\usepackage{newtxtext}

\usepackage{indentfirst}
\input{macros}

\usepackage{algorithm}
\usepackage{algorithmic}
\usepackage{graphicx}
\usepackage{subcaption}
\newcommand{\ours}[1]{\textsc{Poly-EPO}}
\newcommand{\grpo}[1]{\textsc{GRPO}}
\newcommand{\ppo}[1]{\textsc{PPO}}
\newcommand{\cispo}[1]{\textsc{CISPO}}
\newcommand{\polyppo}[1]{\textsc{Poly-PPO}}
\newcommand{\divrl}[1]{\textsc{GRPO}+\textsc{DIV}}
\newcommand{\oursfull}[1]{{Polychromic Exploratory Policy Optimization}}

\input{math_commands.tex}

\title{\Large \textbf{Poly-EPO: Training Exploratory Reasoning Models}}

\author{
  \normalsize {Ifdita Hasan Orney}$^{*}$ \quad \quad {Jubayer Ibn Hamid}$^{*}$ \quad \quad \\ \\
  \normalsize {Shreya S Ramanujam} \quad \quad {Shirley Wu} \quad \quad {Hengyuan Hu} \quad \quad \\ \\
  \normalsize {Noah Goodman} \quad \quad {Dorsa Sadigh} \quad \quad {Chelsea Finn} \\ \\
}

\date{
\small{Stanford University} \\
\vspace{0.5em}
\small $^*$Equal contribution. Correspondence to \texttt{\{ifdi1101, jubayer\}@stanford.edu.}}

\begin{document}
\maketitle

\begin{abstract} 
\noindent 
Exploration is a cornerstone of learning from experience: it enables agents to find solutions to complex problems, generalize to novel ones, and scale performance with test-time compute. In this paper, we present a framework for post-training language models (LMs) that explicitly encourages optimistic exploration and promotes a synergy between exploration and exploitation. The central idea is to train the LM to generate sets of responses that are collectively accurate under the reward function and exploratory in their reasoning strategies. We first develop a general recipe for optimizing LMs with set reinforcement learning (set RL) under arbitrary objective functions, showing how standard RL algorithms can be adapted to this setting through a modification to the advantage computation. We then propose \oursfull{} (\ours{}), which instantiates this framework with an objective that explicitly synergizes exploration and exploitation. Across a range of reasoning benchmarks, we show that \ours{} improves generalization, as evidenced by higher pass@$k$ coverage, preserves greater diversity in model generations, and effectively scales with test-time compute. 
\end{abstract}

\section{Introduction}

Exploration is a fundamental component of learning from experience. First, it is necessary for an agent to discover a successful strategy when learning to solve a highly complex problem, such as identifying a valid proof technique for a difficult theorem. Second, it enables generalization to novel problems by helping the agent build a diverse repertoire of successful strategies; an agent learning to play chess should ideally master multiple opening lines rather than rely on a single strategy that may fail against a different opponent at test time. Third, exploration is also essential for the scalability of test-time compute methods, such as verification, aggregation, and search, whose effectiveness depends on diversity in the model's generations~\citep{wang2023selfconsistencyimproveschainthought, yao2023treethoughtsdeliberateproblem, Besta_2024, zhu2025scalingtesttimecomputellm}. 

\medskip 

While reinforcement learning (RL) has substantially improved the reasoning capabilities of post-trained language models (LMs), effective exploration remains a fundamental challenge. Although pretrained base models exhibit high initial diversity in their generations, RL fine-tuning rapidly collapses this diversity onto a narrow set of high-reward behaviors~\citep{yue2025doesreinforcementlearningreally, cui2025entropymechanismreinforcementlearning,song2025outcomebasedexplorationllmreasoning}. To address this problem effectively, an algorithm should ideally satisfy the following desiderata:
\begin{enumerate}
    \item The algorithm should \emph{encourage optimism under uncertainty}~\citep{LAI19854, 10.1023/A:1013689704352}. In other words, it should assign a positive learning signal to trajectories that attempt novel reasoning strategies, even when those strategies have not yet yielded high reward. This is to ensure that the policy continues to explore promising strategies and broaden its capabilities during training. For example, when attempting to solve an algorithmic problem, an agent exploring a dynamic programming formulation should continue refining that strategy even if its initial implementations fail unit tests, rather than immediately reverting to a familiar but less scalable exhaustive search approach. 
    \item The algorithm should \emph{encourage the model to achieve a synergy between exploration and exploitation}. In other words, the learning signal for each generation should depend on whether that generation helps the model jointly achieve high task performance and explore diverse strategies. This enables the model to learn, through optimization, to balance between these goals. 
    \item The algorithm must be scalable for large scale language model post-training efforts. In particular, the algorithm's sample complexity and computational complexity should not be a significant overhead.  
\end{enumerate}

\medskip

Standard RL algorithms, such as \ppo{}~\citep{schulman2017proximalpolicyoptimizationalgorithms} or \grpo{}~\citep{shao2024deepseekmathpushinglimitsmathematical}, do not explicitly incentivize the optimistic exploration required to master diverse reasoning paths. Prior work has attempted to address this issue through reward shaping, typically by augmenting the canonical reward with a weighted exploration bonus~\citep{cheng2025reasoningexplorationentropyperspective, tuyls2025representationbasedexplorationlanguagemodels, song2025outcomebasedexplorationllmreasoning, dai2025cdecuriositydrivenexplorationefficient, zhang2025countcountsmotivatingexploration}. Broadly, these methods attempt to solve for a given input $x$, the problem $\max_\theta \mathbb{E}_{y \sim \pi_\theta(\cdot \mid x)}\left[ r(x, y) + \lambda d(x, y)\right]$, where $\pi_\theta$ is the parameterized policy, $r(x, y)$ is the task reward and $d(x, y)$ is an exploration bonus, such as an entropy bonus, UCB bonus, or semantic diversity bonus. However, this formulation treats exploitation and exploration as separate objectives and does not explicitly encourage achieving these goals synergistically. A generation can receive high learning signal by being reward-maximizing or by being exploratory, without regard to whether it helps the model strike a balance between exploration and exploitation. Consequently, the method depends on careful tuning or adaptive scheduling of $\lambda$, and even then, does not directly optimize for a persistent balance. 

\medskip

In this paper, we develop a principled and scalable approach for post-training language models that balances exploration and exploitation through \emph{set reinforcement learning} (set RL)~\citep{hamid2026polychromicobjectivesreinforcementlearning}. Set RL optimizes a policy to sample an optimal \emph{set} of generations according to a set-level objective. While \citet{hamid2026polychromicobjectivesreinforcementlearning} introduces the framework and uses it to encourage exploration in simple environments, the formulation is difficult to scale to LM post-training due to its sample complexity. Our contributions are as follows:

\begin{enumerate}
\item \textbf{A general recipe for set reinforcement learning on language models.}
We introduce a general and scalable recipe for optimizing language models under arbitrary set objectives that avoids the complexity of naively sampling sets independently. We show how standard RL algorithms, such as \grpo{}~\citep{shao2024deepseekmathpushinglimitsmathematical}, can be adapted towards this end, providing a practical training paradigm.
\item \textbf{Polychromic Exploratory Policy Optimization.}
Building on this framework, we propose \oursfull{} (\ours{}), which optimizes a \emph{polychromic objective} that scores sets as the product of their average reward and the diversity of the high-level reasoning strategies they contain. We introduce a scalable procedure for estimating the diversity in sampled responses by using a language model judge (LM-judge) to cluster all generations based on their reasoning strategy. 
\item \textbf{Synergistic exploration and exploitation.} We show that our algorithm's advantage function incorporates the covariance between average reward and diversity of strategies across sets. As such, the learning signal of \ours{} inherently harmonizes exploration and exploitation. 
\end{enumerate}

We evaluate our method on reasoning tasks that admit multiple valid solution strategies, as well as on large-scale mathematical reasoning benchmarks. Our experiments show that \ours{} improves generalization by achieving higher pass@$k$ coverage with gains of up to 20\% on mathematical reasoning test sets. The method also scales more effectively with test-time compute under majority voting. Finally, our analysis of training dynamics reveals greater diversity in model generations and broader coverage of training-set problems.

\section{Preliminaries}

In this section, we look at an overview of the standard reinforcement learning framework~\citep{10.5555/3312046} and the set reinforcement learning framework~\citep{hamid2026polychromicobjectivesreinforcementlearning}. 

\subsection{Standard Reinforcement Learning}

\paragraph{Sequential Formulation.}A Markov decision process (MDP) is specified by a state space \(\mathcal{S}\), action space \(\mathcal{A}\), transition dynamics distribution \(p(s' \mid s,a)\), reward function \(r : \mathcal{S} \times \mathcal{A} \to \mathbb{R}\), initial-state distribution \(\rho_0\), and discount factor \(\gamma \in (0,1)\). A policy \(\pi_\theta(a \mid s)\) induces a distribution over trajectories \((s_0,a_0,s_1,a_1,\dots)\), where \(s_0 \sim \rho_0\), \(a_t \sim \pi_\theta(\cdot \mid s_t)\), and \(s_{t+1} \sim p(\cdot \mid s_t,a_t)\). The standard reinforcement learning (RL)~\citep{10.5555/3312046} goal is to solve
\begin{align}
    \max_\theta V(\theta)
    := \max_\theta \mathbb{E}_{\pi_\theta}\!\left[\sum_{t=0}^{\infty}\gamma^t r(s_t,a_t)\right].
\end{align}

\paragraph{Language Model Formulation.}In reinforcement learning with verifiable rewards (RLVR) for language models, we often consider a single-turn generation problem. Let \(x \sim \mathcal{D}\) denote a prompt drawn from a training distribution \(\mathcal{D}\), and let \(y \sim \pi_\theta(\cdot \mid x)\) denote a sampled generation. Then, our goal is to solve 
\begin{align}
    \label{eq: standard_RL_for_LMs}
    \max_\theta \mathbb{E}_{x \sim \mathcal{D}}
    \mathbb{E}_{y \sim \pi_\theta(\cdot \mid x)}
    \bigl[r(x,y)\bigr].
\end{align} We can use the policy gradient, where $A(x, y)$ is the advantage of generation $y$, to optimize this objective:
\begin{align}
    \label{eq: gradient_of_std_RL}
    \nabla_\theta \mathbb{E}_{x \sim \mathcal{D}}
    \mathbb{E}_{y \sim \pi_\theta(\cdot \mid x)}
    \bigl[r(x,y)\bigr] = \mathbb{E}_{x \sim \mathcal{D}}
    \mathbb{E}_{y \sim \pi_\theta(\cdot \mid x)}
    \left[
        \nabla_\theta \log \pi_\theta(y \mid x)\, A(x,y)
    \right]. 
\end{align}

\subsection{Set Reinforcement Learning}

Set reinforcement learning (set RL)~\citep{hamid2026polychromicobjectivesreinforcementlearning} is a framework that generalizes standard RL by assigning reward to sets of sampled actions or generations, rather than to each one independently. Importantly, the reward function couples all the actions or generations under a shared learning signal. We first describe the general sequential formulation and then specialize to the single-turn LM setting studied in this paper.

\paragraph{Sequential Formulation.}
In the idealized sequential formulation, at each visited state \(s\), the policy samples a set of \(n\) actions independently from \(\pi_\theta(\cdot \mid s)\) where $n > 1$ is a fixed integer. Given $a_1,\dots,a_n \overset{\mathrm{i.i.d.}}{\sim} \pi_\theta(\cdot \mid s)$, 
a set-level reward \(f(s,a_{1:n})\) is assigned to the entire set. Repeating this procedure recursively at every timestep generates a \textit{tree} of states visited by our policy. At depth \(t\), the tree contains \(n^t\) states, which we denote by $s_t^{(1)},\dots,s_t^{(n^t)}$. For each such state \(s_t^{(i)}\), let \((a_t)^{(i)}_{1:n}\) denote the \(n\)-tuple of actions sampled from \(\pi_\theta(\cdot \mid s_t^{(i)})\). The goal is to solve the following optimization problem defined using set value function: 
\begin{align}
\label{eq: set_rl_every_state}
\max_\theta V^\sharp_{\pi_\theta}(s;f) = \max_\theta \mathbb{E}_{\pi_\theta}\left[ \sum_{t=0}^{\infty} \sum_{i=1}^{n^t} \gamma^t f\left(s_t^{(i)}, (a_t)^{(i)}_{1:n}\right) \mid s_0 = s \right]. \end{align} In particular, we require that $f$ is an objective such that $\mathbb{E}_{a_{1:n} \sim \pi_\theta(\cdot \mid s)}\left[ f(s, a_{1:n})\right]$ cannot be written as $\mathbb{E}_{a \sim \pi_\theta(\cdot \mid s)}\left[ g(s, a)\right]$ where $g$ is independent of $\theta$ for all policy parameters $\theta \in \Theta$. In other words, we cannot collapse the set RL problem, at any state, to a standard RL one. Intuitively, set RL, in the sequential formulation, encourages the policy to maximize the expected score of the \textit{tree} induced by the policy. In contrast, standard RL encourages the policy to maximize the expected reward of the \textit{single trajectory}. The sequential formulation is computationally impractical in almost all settings because the number of sampled states grows exponentially with depth.

\paragraph{Language Model Formulation.}In single-turn settings, such as in language model training, a more practical formulation is obtained by considering a set-level objective over \(n\) i.i.d.\ generations from the same prompt. Let $
y_{1:n} := (y_1,\dots,y_n)$ where each $ y_i\overset{\mathrm{i.i.d.}}{\sim} \pi_\theta(\cdot \mid x),$
and let $f(x, y_{1:n})$ be our set-level objective. We will assume that \(f\) is symmetric in its arguments \(y_1,\dots,y_n\). Then, set RL aims to solve: 
\begin{align}
    \label{eq: set_RL_for_LMs}
    \max_\theta 
    \mathbb{E}_{x \sim \mathcal{D}}
    \mathbb{E}_{y_{1:n} \sim \pi_\theta(\cdot \mid x)}
    \Bigl[
        f(x,y_{1:n})
    \Bigr].
\end{align}In other words, we use a set-level reward that is shared by all generations within the set, which is natural for training objectives that require generations to collectively exhibit some desirable property. 

Note that we require that the set objective function $f$ is such that we cannot write $\mathbb{E}_{y_{1:n} \sim \pi_\theta(\cdot \mid x)}[f(x,y_{1:n})]$ as $\mathbb{E}_{y \sim \pi_\theta(\cdot \mid x)}[g(x,y)]$ where $g$ is independent of $\theta$ for all policy parameters $\theta \in \Theta$. This distinguishes set RL from the standard RL setting, in which optimization is driven by per-generation rewards. This is an important defining feature because it is always possible to define $g_\theta(x, y) = \mathbb{E}_{y_{2:n}}\left[ f(x, y, y_{2:n})\right]$, so that $\mathbb{E}_{y_{1:n}}\left[f(x, y_{1:n}) \right]$ as $\mathbb{E}_{y}\left[g_\theta(x, y) \right]$. However, this does not reduce the problem to standard RL, since the induced reward $g_\theta$ depends on the policy itself. In particular, it is not an externally specified, policy-independent reward function. This distinguishes the set RL framework from other multi-sample frameworks where one uses leave-one-out estimators to collapse the shared learning signal to an individualized one~\citep{tang2025optimizinglanguagemodelsinference}. 

Using the score-function identity to the joint density
\(
\prod_{i=1}^n \pi_\theta(y_i \mid x)
\), we get the following policy gradient: 
\begin{align}
\label{eq: gradient_of_set_RL}
\nabla_\theta \mathbb{E}_{x \sim \mathcal{D}}
    \mathbb{E}_{y_{1:n} \sim \pi_\theta(\cdot \mid x)}
    \left[
        f(x,y_{1:n})
    \right]
&=
\mathbb{E}_{x \sim \mathcal{D}}
\mathbb{E}_{y_{1:n} \sim \pi_\theta(\cdot \mid x)}
\left[
(f(x,y_{1:n}) - \hat f(x))
\sum_{i=1}^{n}
\nabla_\theta \log \pi_\theta(y_i \mid x)\right]
\end{align} where $\hat{f}(x)$ is a set-level baseline. The defining feature of \cref{eq: gradient_of_set_RL} is that all generations in the sampled set \(y_{1:n}\) share the same scalar learning signal, $f(x,y_{1:n}) - \hat f(x)$. As such, $\hat{f}(x)$ is independent of any proper subset of $y_{1:n}$, which disallows leave-one-out estimators for example. This is fundamentally different from standard RL, where each generation receives its own advantage. In set RL, the objective \(f\) couples the samples in the set, and the gradient assigns equal credit to all elements in the set through a shared set-level signal. Leave-one-out estimators are disallowed because it breaks the shared set-level credit assignment. 

The main practical challenge with \cref{eq: gradient_of_set_RL} is statistical and computational: a naive empirical estimator would require repeatedly sampling full sets of size \(n\) for each prompt in order to estimate both the set reward and the baseline with low variance. In the next section, we show how to construct a more efficient estimator that retains the set-RL learning signal while being practical for large-scale LM training.

\section{Optimizing Language Models with Set RL}
\label{sec: set_rl_for_LMs}

In this section, we discuss a general recipe for optimizing language models (LMs) using the set RL framework under any arbitrary objective function $f$ that is symmetric in its arguments. In particular, we will find a tractable estimator of the set RL policy gradient in \cref{eq: gradient_of_set_RL}. The key ingredient will be a method for constructing the appropriate advantage function for a generation corresponding to the set RL framework. 

Suppose that our set objective, $f(x, y_1,\cdots, y_n)$, is defined over sets of size $n$. If we wanted to use an empirical estimate (e.g., a Monte Carlo estimate) of the gradient in \cref{eq: gradient_of_set_RL}, then we would need to sample several sets of size $n$ conditioned on a single prompt which can be expensive in terms of sample complexity. In this section, we consider how we can instantiate set reinforcement learning using a computationally feasible gradient estimator. 

For a fixed prompt $x$, we first draw $N$ i.i.d.\ generations $
y_1, \ldots, y_N \sim \pi_\theta(\cdot \mid x)$. We require that the number of samples, $N$, is greater than the size of sets, $n$, in the set RL framework. Next, from these samples, we construct $K$ sets, $G_1,\cdots,G_K$, of size $n < N$, in a combinatorial fashion without replacement: 
\[G_j = \{y_{(j_1)}, \ldots, y_{(j_n)}\}, \quad j_1 < \cdots < j_n, \{j_1,\ldots,j_n\} \subseteq \{1,\ldots,N\}.
\]

 The maximum number of sets that can be constructed in this manner is $K = \binom{N}{n}$. Using all $\binom{N}{n}$ sets can reduce the variance of the resulting gradient estimator. When evaluating $f$ is computationally expensive, scoring all $\binom{N}{n}$ sets may be impractical. In this case, we instead sample $K$ sets uniformly (without replacement) from the collection of all $\binom{N}{n}$ possible sets. Either of these choices leads us to an unbiased estimator of the set RL gradient as we will see later. 

Now that we have constructed $K$ sets, for each constructed set $G_j$, we compute the set-level score under the set objective function: 
\[
f(x,G_j) = f\bigl(x, y_{(j_1)}, \ldots, y_{(j_n)}\bigr).
\]

We construct a variance-reduction baseline using the sampled sets: $$ \hat f(x) := \frac{1}{K} \sum_{j=1}^K f(x,G_j).$$ In other words, the baseline is simply an average of the scores of all sets under our objective function $f$. Given the baseline, for each set $G_j$, we can now construct a set advantage function which is simply 
\begin{align}
    \label{eq: set_advantage_of_set}
    \widehat{A^\sharp}(x,G_j; f) = f(x,G_j) - \hat f(x).
\end{align}

So far, we have computed everything at the set-level, i.e. we constructed sets, scored them and computed each set's advantage. We now derive an advantage function at the level of a single generation. 

\paragraph{Marginal set advantage of a single generation.}
Intuitively, we will set the advantage of each generation to be simply the sum of the advantages of all sets containing the generation. Let $\mathcal{G}(y)$ be the collection of all sets that contain the fixed generation $y$. Note that $\mathcal{G}(y) = \{G \in G_{1:K}\mid y \in G \} \subset
G_{1:K}$. We define the marginal set advantage of the generation $y$ to be:
\begin{align}
    \label{eq: set_advantage_of_single_generation}
    \widehat{A_\mathrm{marg}^\sharp}(x, y; f)
    := \frac{1}{|\mathcal{G}(y)|}\sum_{G \in \mathcal{G}(y)}
    A^\sharp(x, G; f).
\end{align} We will soon motivate the construction of this advantage function from the point of view of getting an unbiased estimator for the set RL gradient and, in \S\ref{sec: Analysis}, through a deeper connection between set RL and standard RL. 

\paragraph{Final Estimator.} Using this, we use the following set RL gradient estimator: 
\begin{align}
    \label{eq: set_RL_for_LMs_gradient_estimator}
    \nabla_\theta \widehat{\mathbb{E}}_{x \sim \mathcal{D}, y_{1:n}\sim\pi_\theta(\cdot \mid x)}\left[ f(x, y_{1:n})\right] \propto \widehat{\mathbb{E}}_{x \sim \mathcal{D}, y_{1},\cdots,y_N \sim \pi_\theta(\cdot \mid x)}\left[ \sum_{i=1}^N \nabla_\theta \log \pi_\theta(y_i \mid x) \widehat{A^\sharp_\mathrm{marg}}(x, y_i; f)\right]
\end{align}

Now, we have all the ingredients required for optimizing an LM with respect to the set RL framework. Since the marginal set advantage gives us an advantage function over single generations, we can use existing \textit{standard} RL algorithms, such as \ppo{} \cite{schulman2017proximalpolicyoptimizationalgorithms} or \grpo{} \cite{shao2024deepseekmathpushinglimitsmathematical}, to optimize the set objective by replacing their standard advantage, $A(x, y_i)$, with $\widehat{A_\mathrm{marg}^\sharp}(x, y_i)$. The pseudocode for our general recipe is provided in \cref{alg: general_set_rl_pseudocode}.

Note that our estimator is an unbiased estimator of the true policy gradient of set RL (up to a scaling factor). This is true for both the case where we construct all $\binom{N}{n}$ sets and the case where we sample $K$ sets uniformly without replacement. We show this in the following proposition:  
\begin{prop}
\label{prop:set_rl_u_stat_unbiased}
Fix a prompt $x$, and let $y_1,\cdots,y_N \overset{\mathrm{i.i.d.}}{\sim} \pi_\theta(\cdot \mid x)$ be our independently sampled $N$ generations and let $f : \mathcal{X} \times \mathcal{Y}^{\oplus n} \rightarrow \mathbb{R}$ be our set objective. Then, $$\mathbb{E}[
\sum_{i=1}^N
\nabla_\theta \log \pi_\theta(y_i \mid x)\,
\widehat{A_{\mathrm{marg}}^\sharp}(x,y_i; f)]
=
M  \nabla_\theta
\mathbb{E}_{y_{1:n} \sim \pi_\theta(\cdot \mid x)}
\bigl[f(x,y_{1:n})\bigr],$$ where $M \in \mathbb{R}_{>0}$ is a constant scaling factor that depends on the number of sets we construct. Consequently, after also taking expectation over $x \sim \mathcal{D}$ and scaling the learning rate, the estimator is an unbiased estimator of the set RL gradient.
\end{prop}

The proof of this is provided in \S\ref{appendix: proofs}. The main idea of the proof is that constructing the sets in a combinatorial fashion causes our proposed gradient to be a U-statistic estimator of the set RL gradient which is unbiased. In \S\ref{subsec: analysis_of_general_recipe}, we revisit this question from a different point of view and justify why set RL can be implemented using standard RL algorithms. 

\begin{algorithm}[t]
\caption{Gradient of On-Policy Implementation of Set RL}
\label{alg: general_set_rl_pseudocode}
\begin{algorithmic}[1]
\REQUIRE Set objective $f$, Policy $\pi_\theta$, batch of inputs $B$, number of rollouts $N$, set size $n$
\FOR{each input $x \in B$}
    \STATE Sample $y_1, \dots, y_N \sim \pi_\theta(\cdot \mid x)$
    \STATE // {\small\textit{Construct sets either by enumerating all $\binom{N}{n}$ subsets or by uniformly sampling $K$ subsets from them}}
    \STATE Construct sets $G_1,\cdots,G_K$ from $\{ y_{1:N}\}$ without replacement 
    \STATE Compute score of each set $\{f(x, G_l)\}_{l=1}^K$ and baseline $b = \frac{1}{K}\sum_{l=1}^K f(x, G_l)$
    \STATE Compute set advantage of each set $\widehat{A^\sharp}(x, G_l; f) = f(x, G_l) - b$ for $l = 1,\cdots,K$
    \STATE // {\small \textit{Let $\mathcal{G}(y)$ be all the sets in $G_1,\cdots,G_K$ that contain $y$} }
    \STATE Compute marginal set advantage $\widehat{A^\sharp_\mathrm{marg}}(x, y_j; f) := \frac{1}{|\mathcal{G}(y_j)|} \sum_{G \in \mathcal{G}(y_j)} \widehat{A^\sharp}(x, G)$ for $j =1,\cdots,N$
    \STATE $\hat{g}(x) \leftarrow \sum_{j=1}^N  \nabla_\theta \log \pi_\theta(y_j \mid x) \cdot \widehat{A^\sharp_\mathrm{marg}}(x, y_j; f)$
\ENDFOR
\STATE $\hat{g} \leftarrow \frac{1}{|B|} \sum_{x \in B} \hat{g}(x)$
\RETURN $\hat{g}$
\end{algorithmic}
\end{algorithm}

\section{Polychromic Exploratory Policy Optimization}
\label{sec: polyepo}

In this section, we introduce our method, \oursfull{} (\ours{}). We use the set RL framework discussed in \S\ref{sec: set_rl_for_LMs} to optimize an objective that explicitly encourages the model to balance exploration and exploitation. Concretely, we want the model to increase the likelihood of generations that are not only accurate with respect to our reward function, but also explore diverse reasoning strategies. We first discuss the specific objective we optimize and how to implement it in a scalable manner. Then, we put all the pieces together to present the final algorithm.  

\paragraph{Polychromic Objective.} The preceding development is agnostic to the specific choice of set objective. Now, we discuss the construction of our algorithm for fine-tuning an LM by explicitly encouraging exploration. We aim to train an LM to reason by exploring diverse reasoning strategies whilst also increasing accuracy on the training set. We use the polychromic objective adapted to the LM reasoning setting~\citep{hamid2026polychromicobjectivesreinforcementlearning}: \begin{align}
    \label{eq: polychromic_objective}
    f_\mathrm{poly}(x, y_1,\cdots,y_n) = \frac{1}{n}\sum_{i=1}^n r(x,y_i)\cdot d(x,y_1,\cdots,y_n),
\end{align} where $d(x,y_1,\cdots,y_n)$ is a function that measures the diversity of the set of generations, $y_{1:n}$. Notably, due to the advantage being shared by generations in a set, under polychromic objectives, an unsuccessful generation (i.e. a generation $y$ such that $r(x, y) = 0$) can still receive a positive set-advantage (i.e., $A^\sharp_\mathrm{marg}(x, y; f_\mathrm{poly}) > 0$) if it is highly exploratory. 

Intuitively, the polychromic objective assigns a score to a set of generations that is dependent on both its ability to explore diverse strategies and exploit the reward signal. Since this is a product and since each of these two factors is lower bounded by 0, the set must optimize both exploration and exploitation, and cannot ignore one for the other.

\paragraph{Diversity Function.} We now discuss how we measure diversity within a set of generations. Several considerations are important. First, computing the diversity score $d(x, y_{1:n}$ for each set must be computationally efficient. For variance reduction purposes, we construct all $K = \binom{N}{n}$ sets of size $n$ from the $N$ generations sampled per prompt; consequently, if evaluating each set under the polychromic objective is expensive, the runtime of the algorithm grows rapidly as $N$ increases. Second, the diversity measure should be broadly applicable across domains. Language models are post-trained on tasks spanning mathematics, coding, instruction-following, and many others, so we require a notion of diversity that transfers naturally across settings.

Our main ingredient is a language-model judge (LM-judge) used to cluster responses sampled from the policy according to their underlying reasoning strategy or semantic approach. Because LM-judges are already integrated into modern post-training pipelines—including settings such as proof generation and other non-verifiable tasks—this design allows \ours{} to leverage infrastructure that is already widely deployed. Moreover, clustering responses by semantic similarity is typically a substantially easier problem than solving the original task itself, which allows our method to exploit the generator--verifier gap.

Recall that, for each prompt $x$, we first sample $N$ independent generations and then construct sets of size $n$ from them. We begin by taking all $N$ generations associated with the same prompt and querying an LM-judge to cluster them according to the semantic similarity of their reasoning strategies. Let $\mathcal{C}(y_i) \in {1,\dots,N})$ denote the cluster assignment of generation $y_i$. We then define the diversity of a set $({y_{1},\dots,y_{n}})$ as
\begin{align}
\label{eq: diversity_function}
d\bigl(x, y_{1:n} \bigr)
=
\frac{\text{number of distinct clusters represented in }y_{1:n}}{n}
=
\frac{\left| \{ \mathcal{C}\left(y_{1}\right), \dots, \mathcal{C}\left(y_{n}\right) \} \right|}{n}.
\end{align}
By construction, $d(x, y_{1:n}) \in [0,1]$. The pseudocode for computing the polychromic objective score for each set is provided in \cref{alg: polychromic_scoring}.

To improve judge reliability, we use in-context learning to steer the LM-judge toward clustering generations based on their underlying strategy—capturing both high-level and low-level reasoning structure—rather than superficial differences such as tone, phrasing, or writing style. The full judge prompt is provided in the appendix (see \S\ref{appendix: lm-judge_and_diversity_function}), where we detail the rubric, few-shot exemplars, and other design choices used to improve clustering quality.

\begin{algorithm}[t]
\caption{Compute Polychromic Objective}
\label{alg: polychromic_scoring}
\begin{algorithmic}[1]
\REQUIRE Prompt $x$, generations $\{y_{1:N}\}$, sets $\{G_{1:K}\}$, judge $\mathcal{M}$, reward $r$
\STATE $\{\mathcal{C}(y_1), \dots, \mathcal{C}(y_N)\} \leftarrow \mathcal{M}(x, \{y_{1:N}\})$ // \textit{Step 1: Semantic Clustering using LM-Judge}
\FOR{$l = 1$ to $K$}
    \STATE $U_l \leftarrow \{ \mathcal{C}(y) \mid y \in G_l\}$ \textbf{and} $\bar{r}_l \leftarrow \frac{1}{n} \sum_{y \in G_l} r(x, y)$
    \STATE $f_\mathrm{poly}(x, G_l) \leftarrow \bar{r}_l \cdot \frac{|U_l|}{n}$ 
\ENDFOR
\RETURN $\{f_\mathrm{poly}(x, G_l)\}_{l=1}^K$
\end{algorithmic}
\end{algorithm}

\paragraph{Final Algorithm.} The final algorithm, \oursfull{} (\ours{}), computes the marginal set-RL advantage of each generation under the polychromic objective and then plugs this signal into any standard RL update rule. In designing \ours{}, we prioritize simplicity and scalability above all else. We sample $N$ generations independently per prompt, exactly as in standard RL algorithms such as \grpo{} and \cispo{}. Moreover, as in \grpo{}, we assign the same advantage to every token within a generation. This amounts to instantiating the general set-RL recipe in \cref{alg: general_set_rl_pseudocode}, with set scores computed using the procedure in \cref{alg: polychromic_scoring}.

\section{Analysis}
\label{sec: Analysis}

\subsection{Analysis of the General Recipe}
\label{subsec: analysis_of_general_recipe}

In this section, we ask the following question: \textit{why can standard RL algorithms be used to optimize set RL by simply replacing the advantage to be the marginal set advantage as shown in \cref{eq: set_advantage_of_single_generation}?} First, we recall a result analyzing logit shifts in standard reinforcement learning. Suppose, we parametrize our policy (in this case, the LM) as a softmax distribution. Then, \cite{cui2025entropymechanismreinforcementlearning} showed that, after one gradient update with learning rate $\alpha$, the shift in probability mass can be expressed as:
\begin{align}
    \label{eq: logit_shift_in_RL}
    \log \pi_\theta^{k+1}(y \mid x) - \log \pi_\theta^{k}(y \mid x) = \alpha \pi_{\theta}^k(y \mid x) A(x, y; \pi_\theta^k)
\end{align}
where $A(x, y; \pi_\theta^k)$ is the advantage of the generation with respect to the policy $\pi_\theta^k$ under the standard RL algorithm. 

Now, we look at logit shifts in the set reinforcement learning framework. The result follows from the observation in \cite{hamid2026polychromicobjectivesreinforcementlearning} on the \textit{implicit individual credit} that a generation receives in the set RL framework. In adapted form, the logit shift looks as follows: 

\begin{lem}
    \label{lem: probability_mass_shift_in_setRL}
    In set reinforcement learning, assuming the objective function $f$ is symmetric in its arguments and the learning rate is $\alpha$, the shift in probability mass on a fixed generation $y$ after one step gradient update from $\pi_\theta^k$ to $\pi_\theta^{k+1}$ can be written as:
    \begin{align}
        \label{eq: logit_shift_in_set_RL}
        \log \pi_\theta^{k+1}(y \mid x) - \log \pi_\theta^{k}(y \mid x) 
        &= \alpha n\pi_{\theta}^k(y \mid x) \left[ \mathbb{E}_{y_{2:n}\sim \pi_\theta^k(\cdot \mid x)}\left[f(x, y, y_{2:n}) \right] - \mathbb{E}_{y_{1:n}\sim \pi_\theta^k(\cdot \mid x)}\left[f(x,  y_{1:n}) \right]\right]- \alpha C(\theta^k).
    \end{align}
\end{lem}

The proof can be found in \S\ref{appendix: proofs}. This result tells us that, after one update using set RL, the log likelihood of a generation being sampled changes according to the contribution of the generation to the score of sets that contain it. In particular, the marginal set advantage builds the bridge between these two frameworks. Concretely, comparing \cref{eq: logit_shift_in_RL} with \cref{eq: logit_shift_in_set_RL}, we can simply define the marginal set advantage of a generation as 
\begin{align}
    \label{eq: marginal_set_adv_def}
    A_\mathrm{marg}^\sharp(x, y; f, \pi_\theta^k) := \mathbb{E}_{y_{2:n}\sim \pi_\theta^k(\cdot \mid x)}\left[f(x, y, y_{2:n}) \right] - \mathbb{E}_{y_{1:n}\sim \pi_\theta^k(\cdot \mid x)}\left[f(x,  y_{1:n}) \right].
\end{align}This enables us to use any standard RL algorithm to update our policy using set RL by simply replacing the advantage function with our marginal set advantage, since in doing so, the logit shift is equivalent. 

This derivation points us to a fundamental difference between standard RL and set RL. In the standard RL setting, the logit shift for a particular generation $y$ is simply governed by $A(x, y)$ which is a measure of how much better the generation is compared to the baseline in terms of returns. In contrast, in the set RL framework, the logit shift for generation $y$ is governed by the contribution of $y$ to sets' scores under the objective $f$, i.e., the marginal set advantage of $y$. 

Apart from being useful for deriving our algorithm, this also helps us understand what a set objective function explicitly encourages. As an example, consider the pass@$n$ objective $f_\mathrm{pass@n}(x, y_{1:n}) = \max_{i \in \{1,\cdots,n\}} r(x, y_i)$ \citep{tang2025optimizinglanguagemodelsinference, walder2025passkpolicyoptimizationsolving, chen2025passktrainingadaptivelybalancing}. In the set RL framework, the marginal set advantage of any generation that achieves reward 0 is $A_\mathrm{marg}^\sharp(x, y_\mathrm{incorrect}) = (1-p)^n - (1-p)^{n-1},$ where $p$ is the probability of sampling a generation that achieves reward 1. As such, the marginal set advantage of an incorrect generation is non-positive. On the other hand, the marginal set advantage of a correct generation is $A_\mathrm{marg}^\sharp(x, y_\mathrm{correct}) = (1-p)^n,$ which is always non-negative.

Observe that the marginal set advantage of a generation gives us insight into what the algorithm encourages precisely because of \cref{eq: logit_shift_in_set_RL}. We will now use this as our apparatus to understand what different objectives optimize for in the next section.

\subsection{Exploration-Exploitation Synergy}
\label{subsec: analysis_of_adv_functions}

In this section, we will analyse how the polychromic objective encourages the policy to synergistically explore and exploit. To do so, we analyze which generations receive increased probability mass under the polychromic objective and compare this behavior with alternative approaches for encouraging exploration in reinforcement learning. This analysis clarifies the types of model behaviors each method directly incentivizes.

\paragraph{Advantage under Polychromic Objective in Set RL.} Recall that in set reinforcement learning, \cref{lem: probability_mass_shift_in_setRL} shows that a generation $y$ receives increased probability mass (i.e., $\log \pi_\theta^{k+1}(y \mid x) - \log \pi_\theta^k(y \mid x) > 0$) if and only if it has a positive marginal set advantage and non-zero probability under the rollout policy. Consider a fixed generation $y \in \mathcal{Y}$ with $\pi_\theta(y \mid x) > 0$ and the polychromic objective $f_\mathrm{poly}$. The marginal set advantage of $y$, denoted $A_\mathrm{marg}^\sharp(x, y; f_\mathrm{poly}, \theta)$, can be decomposed as follows:

{\begin{align}
    \label{eq: marginal_adv_set_rl_polychromic_objective}
    & A_\mathrm{marg}^\sharp(x, y; f_\mathrm{poly}, \theta) \nonumber \\
    & =
    \underbrace{
    ( \frac{1}{n}r(x, y)  + \frac{n-1}{n}\mathbb{E}_{Y}\!\left[ r(x, Y) \right])
    \mathbb{E}_{Y_{2:n}}\!\left[ d(x, y, Y_{2:n}) \right]
    }_{\substack{\text{\textit{Term 1}: Mean reward of sets containing }y \text{ } \times \text { Mean diversity of sets containing }y}}
    - \quad 
    \mathbb{E}_{Y}\!\left[ r(x, Y) \right] \mathbb{E}_{Y_{1:n}}\!\left[ d(x,Y_{1:n}) \right] \nonumber \\
    &\quad
    + 
    \underbrace{\mathrm{Cov}_{Y_{2:n}}\! (
    \frac{1}{n}r(x, y) + \frac{1}{n}\sum_{i=2}^n r(x, Y_i),
    d(x,y, Y_{2:n})
    )}_{\substack{\text{\textit{Term 2}: Exploration--Exploitation synergy in sets containing }y}} \quad 
    - \quad \mathrm{Cov}_{Y_{1:n}}\!\left( r(x, Y_1), d(x, Y_{1:n})\right).
\end{align}}

This decomposition reveals how set RL with polychromic objective evaluates each generation. We now interpret the two terms that depend on the fixed generation $y$.

\begin{itemize}
\item \textit{Term 1} reflects the expected contribution of generation $y$ to the product of reward and diversity in sets that contain it. Importantly, this term depends on both the reward and diversity contributions of other generations in the set. Consequently, even if $y$ itself does not contribute reward (or diversity), this term can still be positive provided other sampled generations contribute reward (or diversity) in expectation. Thus, generations that explore optimistically receive a positive learning signal even if they failed to yield high rewards. 

\item \textit{Term 2} captures how the presence of generation $y$ affects the covariance between reward and diversity across sets containing $y$. Intuitively, this term favors generations that enable sets to simultaneously achieve high reward and high diversity. Thus, generations that help achieve a synergy between exploration and exploitation receive positive learning signal.
\end{itemize}

This directly reflects the two desiderata that motivate our method. First, the objective supports \emph{optimistic exploration}: trajectories that attempt novel reasoning strategies can receive positive learning signal even before they become individually successful. Second, the objective \emph{intrinsically encodes the exploration--exploitation trade-off}: the advantage assigned to a trajectory depends in part on whether it helps the model form sets that jointly achieve high reward and high diversity, rather than on a manually tuned weighting coefficient. In this sense, the policy update itself internalizes the balance between exploration and exploitation.

\paragraph{Advantage under Reward Shaping in Standard RL.} Next, consider a standard RL approach with an exploration bonus. In this framework, the objective is $$\mathbb{E}_{y \sim \pi_\theta (\cdot \mid x) }\left[ \Tilde{r}(x, y) \right] = \mathbb{E}_{y \sim \pi_\theta (\cdot \mid x) }\left[ r(x, y) + \lambda \cdot d(x, y)\right]$$ where $d(x,y)$ measures the diversity or exploration contribution of generation $y$. We assume $d(x,y)\in[0,1]$, although the analysis does not depend on the specific implementation of $d$. Under this objective, a generation receives increased probability if its standard RL advantage is positive~\citep{cui2025entropymechanismreinforcementlearning}. The advantage can be written as \begin{align}
\label{eq: adv_std_rl_with_divbonus_added}
A(x, y; \tilde r, \theta) = 
r(x, y) - \mathbb{E}_Y[r(x,Y)]
+
\lambda\big(d(x,y) - \mathbb{E}_Y[d(x,Y)]\big).
\end{align}

We now highlight several fundamental differences between the marginal set advantage under \ours{} (\cref{eq: marginal_adv_set_rl_polychromic_objective}) and the advantage induced by reward-shaped standard RL (\cref{eq: adv_std_rl_with_divbonus_added}). Let $p$ denote the probability that the model samples a correct generation for prompt $x$.

\begin{itemize}
\item Consider a generation $y$ with $r(x,y)=0$ (an incorrect response) but high diversity contribution, i.e., $d(x,y)$ is large. Under the reward-shaped RL objective, $y$ receives positive advantage only if $d(x, y) \geq 
\frac{p}{\lambda} + \mathbb{E}_Y[d(x, Y)]$. Thus, the exploration coefficient $\lambda$ critically determines whether exploratory but incorrect generations are encouraged. Moreover, as model accuracy improves (i.e., $p$ increases), this threshold becomes harder to satisfy (for e.g, with $\lambda = 0.5$, if $p > 0.5$, then this bound is impossible to satisfy), making it increasingly difficult for incorrect but diverse generations to receive positive advantage.

In contrast, under set reinforcement learning, as a model's accuracy improves, exploratory generations can still receive positive learning signal even when incorrect. Because the set objective shares credit among all generations in a set, incorrect but diverse generations can benefit from the presence of correct generations in the same set. Consequently, the incentive for exploration does not vanish as success rate improves and does not rely on tuning a coefficient such as $\lambda$. 

\item Furthermore, \ours{} naturally assigns advantage based on how a generation contributes to achieving a synergy between exploration and exploitation over sets. This behavior arises from two components: the set RL framework and the multiplicative structure of the polychromic objective. Notably, this effect does not arise in the reward-shaped RL term and would not arise if one simply applied standard RL with a reward of the form $r(x,y)\cdot d(x,y)$, as considered in some prior work~\citep{li2025jointlyreinforcingdiversityquality}. In such formulations, the covariance term in \cref{eq: marginal_adv_set_rl_polychromic_objective} does not appear, and the learning signal cannot explicitly encourage generations that help coordinate exploration and exploitation across sets.
\end{itemize}

\section{Experiments}

We now empirically evaluate \ours{}. In \S\ref{subsec: experiments_maths_reasoning}, we train and evaluate on mathematical reasoning benchmarks. In \S\ref{subsec: experiments_infinite_strategy}, we study synthetic domains with verifiable rewards and infinitely many reward-maximizing strategies.

\subsection{Mathematical Reasoning}
\label{subsec: experiments_maths_reasoning}

\begin{figure}[t]
    \centering    
    \includegraphics[width=1.0\linewidth]{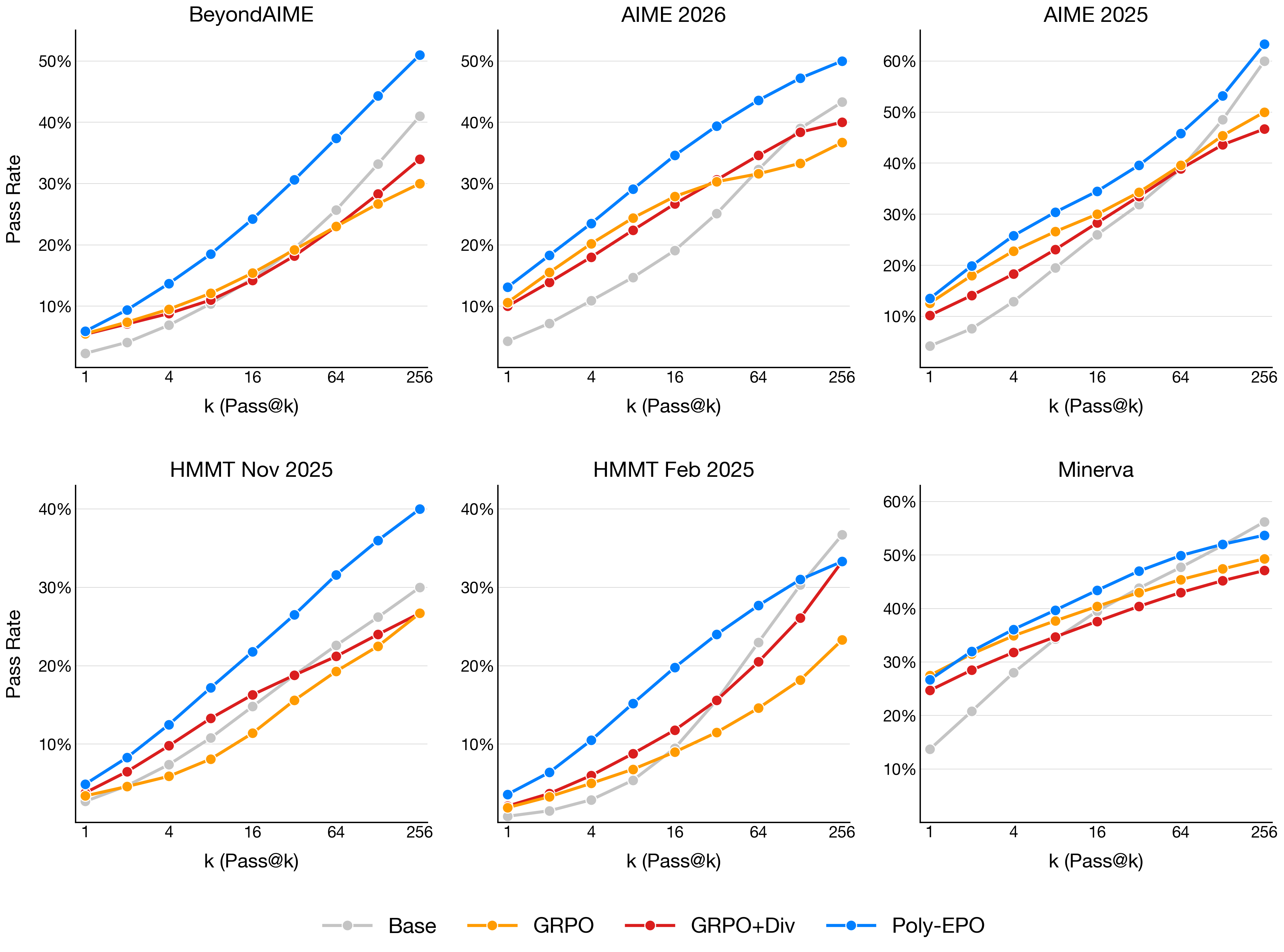}
    \caption{\textbf{Pass@$k$ evaluations on test sets}. The x-axis is the number of attempts $k$ used in the evaluation while the y-axis is the coverage of the test set. The base model used here is \texttt{Qwen3-4B-Base} and the LM-judge used by \ours{} and \divrl{} for clustering responses is \texttt{Qwen-3-4B-Instruct}.}
    \label{fig:maths_reasoning_pass@k}
\end{figure}

We evaluate \ours{} on large-scale mathematical reasoning tasks. Our goal is to test whether the exploration induced by \ours{} translates into better reasoning performance, in particular through improved generalization to unseen problems and more effective use of test-time compute. We use \texttt{Qwen-3-4B-Base} as our pretrained base model and train on \texttt{POLARIS-53k}~\citep{Polaris2025}, a dataset of mathematical reasoning problems. We use 128 prompts per batch, 8 rollouts per prompt, and 2 epochs of RL training. We evaluate on a broad suite of held-out benchmarks: BeyondAIME~\citep{seed2025seed15thinkingadvancingsuperbreasoning}, AIME 2026, AIME 2025, HMMT November 2025~\citep{balunovic_srimatharena_2025}, HMMT February 2025~\citep{balunovic_srimatharena_2025}, and Minerva~\citep{lewkowycz2022solvingquantitativereasoningproblems}. 

We compare against \grpo{}~\citep{shao2024deepseekmathpushinglimitsmathematical} and against a standard RL baseline with an exploration bonus, which we denote by \divrl{}. Its objective is \begin{align}
    \label{eq: baseline_exploration_objective}
    \mathbb{E}_{x \sim \mathcal{D}}
    \mathbb{E}_{y \sim \pi_\theta(\cdot \mid x)}
    \left[
        r(x,y) + \lambda\, d(x,y \mid y_{1:N})
    \right].
\end{align}Here, $d(x,y \mid y_{1:N})$ is a diversity bonus defined using the same LM-judge-based clustering procedure as in \ours{}. Specifically, given rollouts $y_1,\dots,y_N$ sampled for prompt $x$, we first cluster these generations. Let $\mathrm{Cluster}(y)$ be the set of generations that belong to the same cluster as $y$. Then define
\[
d(x,y \mid y_{1:N})
=
\frac{\frac{N}{\left| \mathrm{Cluster}(y) \right|} - 1}{N -1}.
\]
Thus, the bonus assigned to a generation decreases as the size of its cluster increases. Note that $d(x, y \mid y_{1:N}) \in [0, 1]$. This diversity bonus is non-parametric and, in particular, non-differentiable with respect to the policy parameters. Note that the objective is not a set RL objective. Conceptually, \divrl{} is similar to prior approaches that add exploration bonuses to the reward~\citep{song2025outcomebasedexplorationllmreasoning, tuyls2025representationbasedexplorationlanguagemodels}. We optimize \divrl{} using \grpo{}. 

We instantiate \ours{} and \divrl{} using \texttt{Qwen-3-4B-Instruct} as the LM-based judge for clustering generations. Note that the judge must possess strong instruction-following capabilities, as it is required to cluster responses according to a provided rubric that specifies grouping based on the underlying high-level reasoning strategy. Note that both \ours{} and \divrl{} use the same clustering method - the only difference is the diversity function that uses the clustering; \ours{} uses the clustering to score diversity of sets whereas \divrl{} computes the diversity contribution of single generations. Additional implementation details are provided in \S\ref{appendix: implementation-details}.

\paragraph{Pass@$k$ Evaluations.}We summarize the pass@$k$ results in \cref{fig:maths_reasoning_pass@k}. Overall, the pass@$k$ performance of \ours{} improves more strongly as $k$ increases, indicating that it produces generations with higher effective diversity. In contrast, \grpo{} and \divrl{} exhibit substantial pass@$k$ degradation; indeed, the pretrained base model (\texttt{Qwen-3-4B-Base}) begins to outperform them as early as $k=32$. This improved diversity under repeated sampling is especially valuable in domains such as mathematics, where one can pair generation with strong external verifiers such as Lean. 

\begin{figure} [t]
    \centering
    \includegraphics[width=0.9\linewidth]{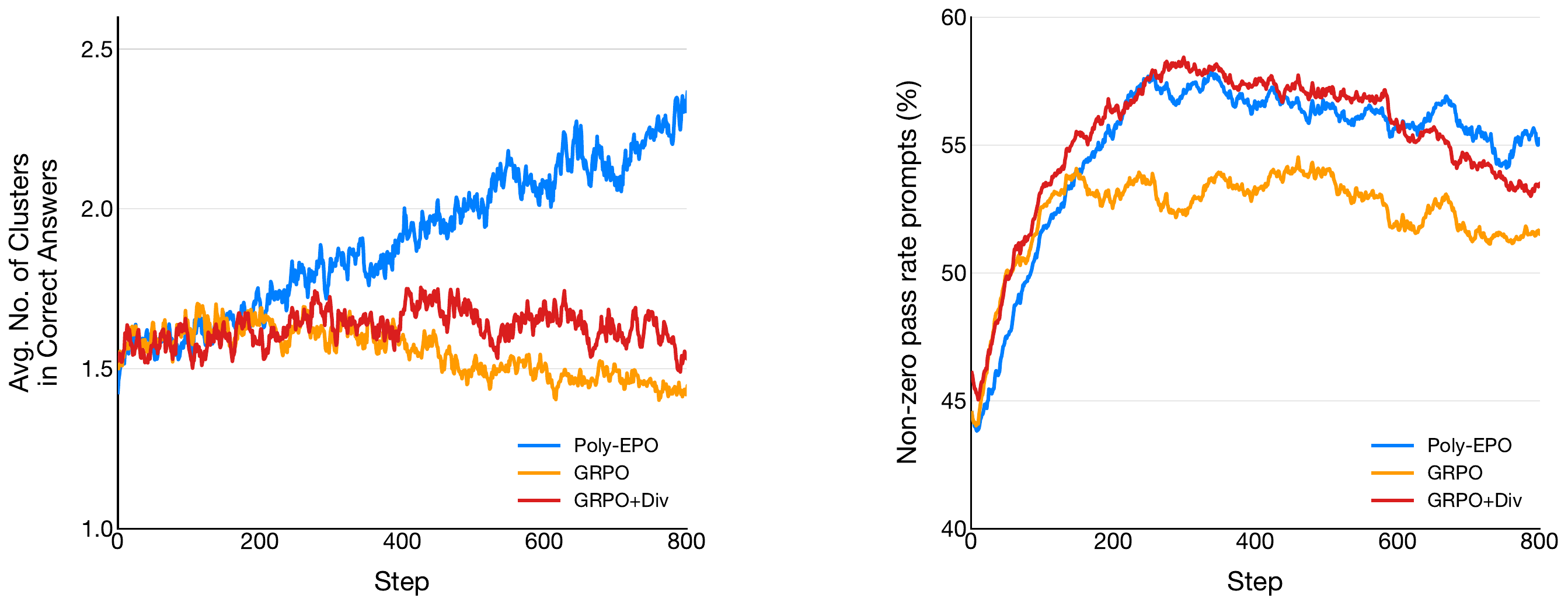}
    \caption{\textbf{Training dynamics on \texttt{POLARIS-53k}.} \textbf{Left:} Average number of unique reasoning-strategy clusters among the correct generations sampled for each prompt during training. Clusters are assigned by \texttt{Qwen-3-4b-Instruct}, which groups generations according to their high-level reasoning strategy. Higher values indicate a greater diversity of successful reasoning approaches. \textbf{Right:} Fraction of training prompts for which the model generates at least one correct rollout among its sampled responses during training.}
    \label{fig: training_dynamics_maths_reasoning}
\end{figure}

\paragraph{Training Dynamics.}We now analyze the training dynamics. We first measure the average number of unique reasoning-strategy clusters among the correct generations sampled for each prompt, where cluster assignments are produced by the model \texttt{Qwen-3-4b-Instruct}. The results are shown in \cref{fig: training_dynamics_maths_reasoning} (left). For \grpo{}, this clustering is used purely for post-hoc evaluation and logging; no diversity signal is used during training. We observe that under \grpo{}, the number of clusters begins to decline after roughly 200 training steps, indicating a contraction in the diversity of successful strategies. Under \divrl{}, the metric remains approximately flat relative to its initial value, suggesting limited expansion of the successful strategy set. In contrast, \ours{} steadily increases the number of unique clusters to substantially higher levels than either baseline, indicating more effective exploration and the discovery of a broader range of successful reasoning approaches.

\begin{figure*}[t]
    \centering

    \begin{subfigure}{0.49\textwidth}
        \centering
        \includegraphics[width=\linewidth]{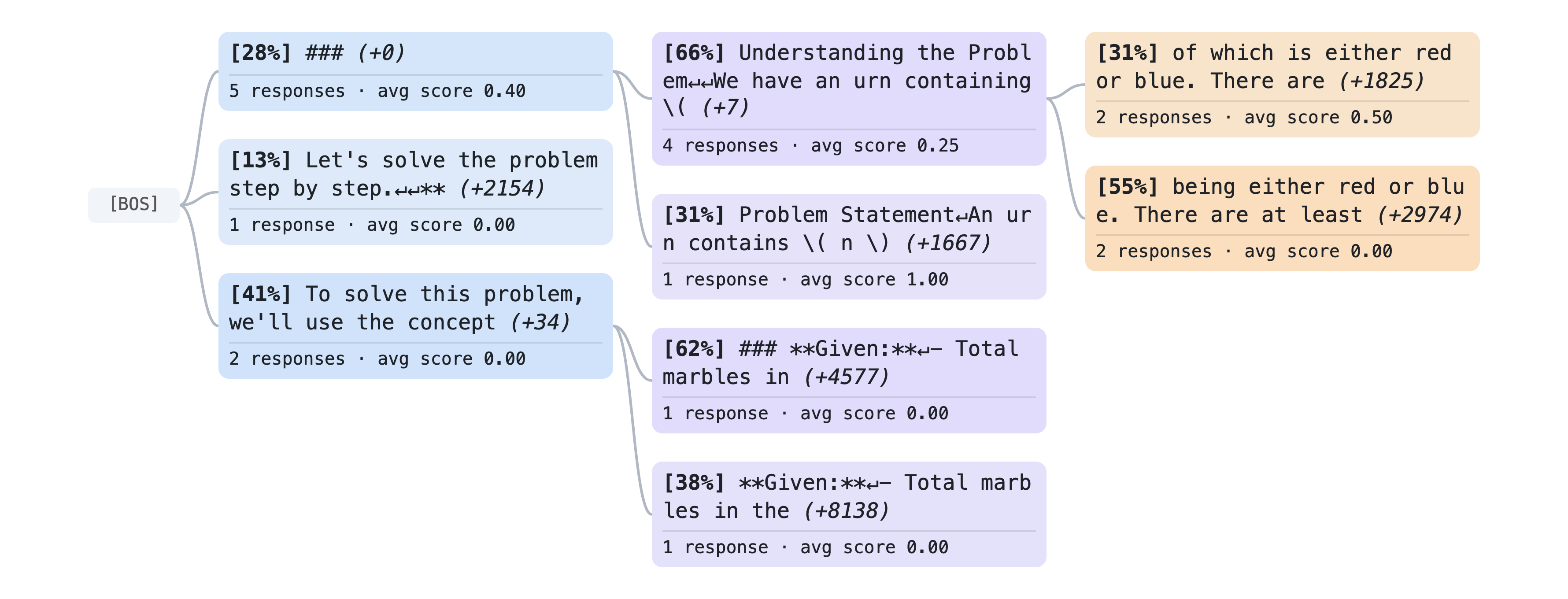}
        \caption{\ours{}}
        \label{fig:rollout-poly_aime26}
    \end{subfigure}
    \hfill
    \begin{subfigure}{0.49\textwidth}
        \centering
        \includegraphics[width=\linewidth]{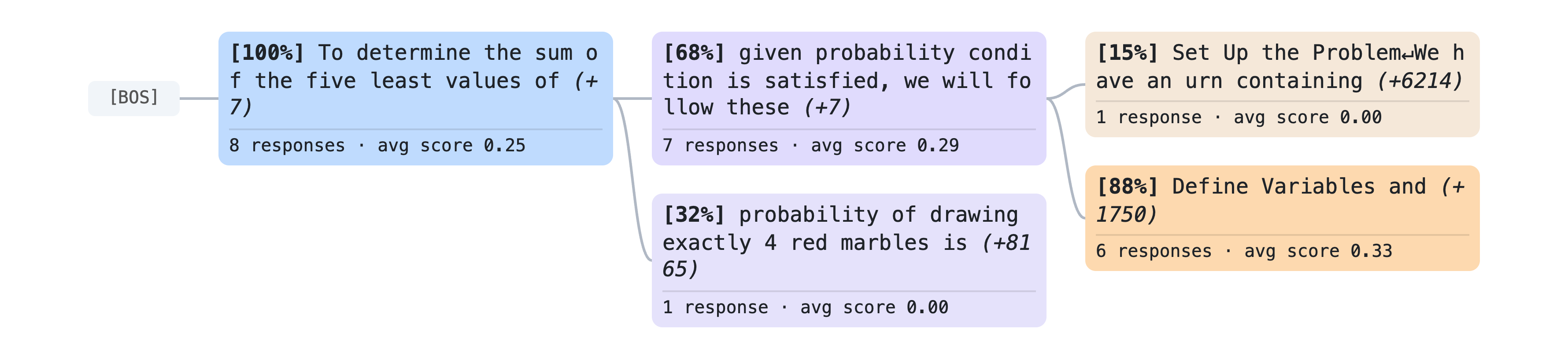}
        \caption{\grpo{}}
        \label{fig:rollout-grpo_aime26}
    \end{subfigure}

    \vspace{1.2em}

    \begin{subfigure}{0.31\textwidth}
        \centering
        \includegraphics[width=\linewidth]{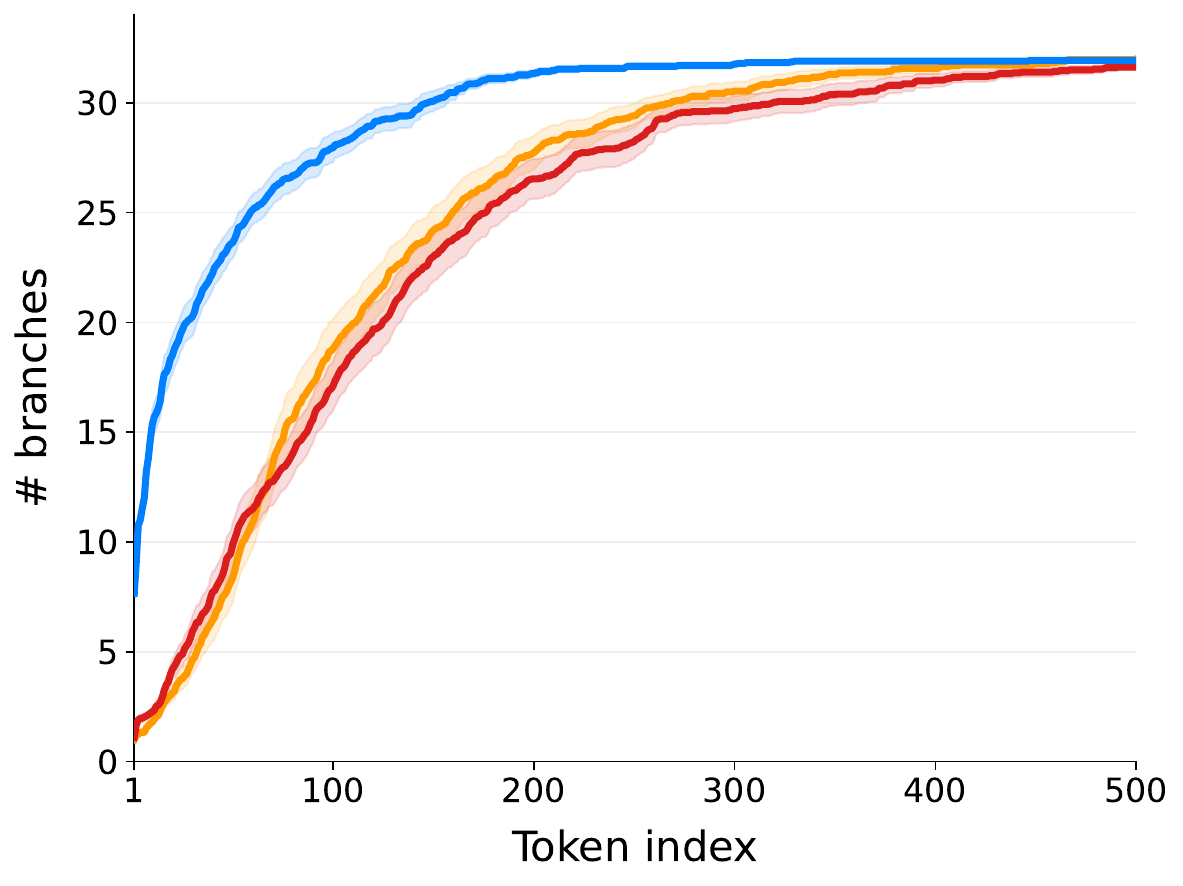}
        \caption{AIME 2026}
        \label{fig:branches-aime26}
    \end{subfigure}
    \hfill
    \begin{subfigure}{0.31\textwidth}
        \centering
        \includegraphics[width=\linewidth]{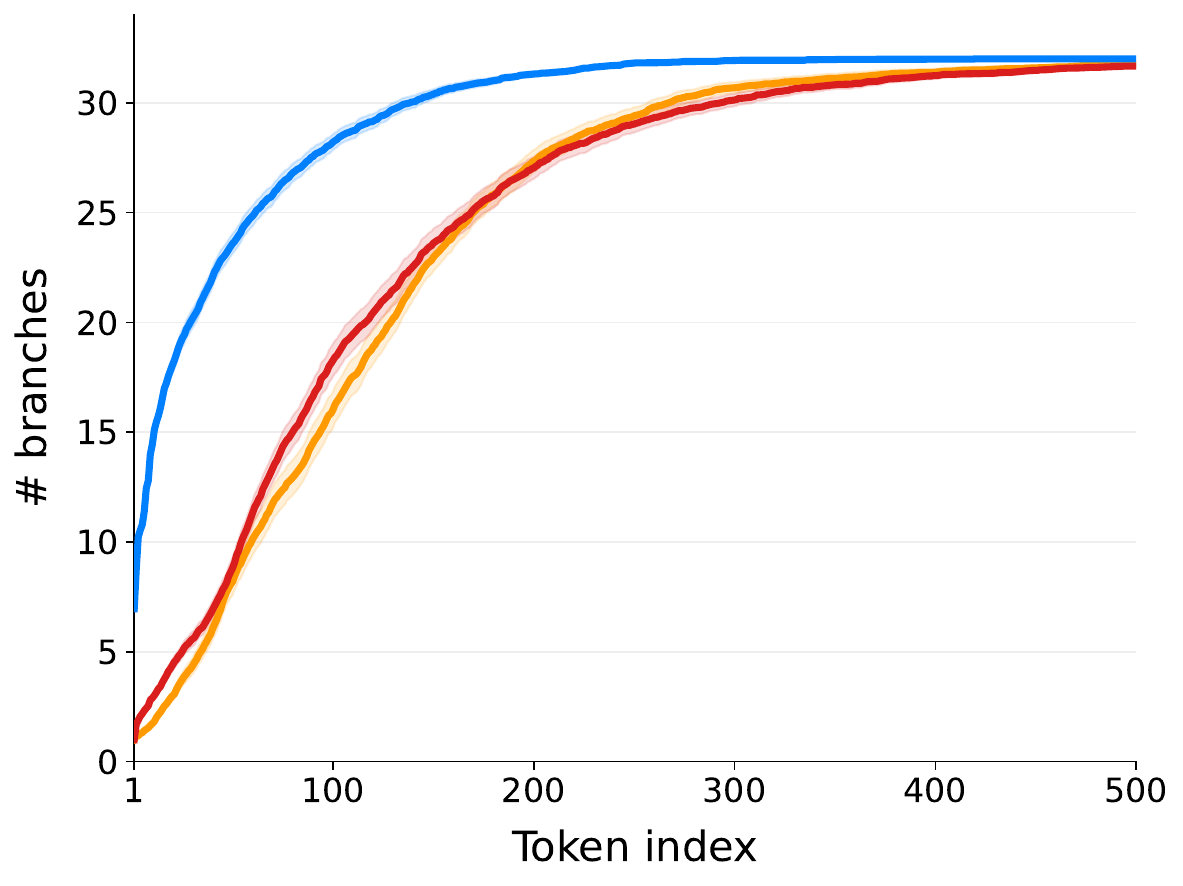}
        \caption{BeyondAIME}
        \label{fig:branches-beyondaime}
    \end{subfigure}
    \hfill
    \begin{subfigure}{0.31\textwidth}
        \centering
        \includegraphics[width=\linewidth]{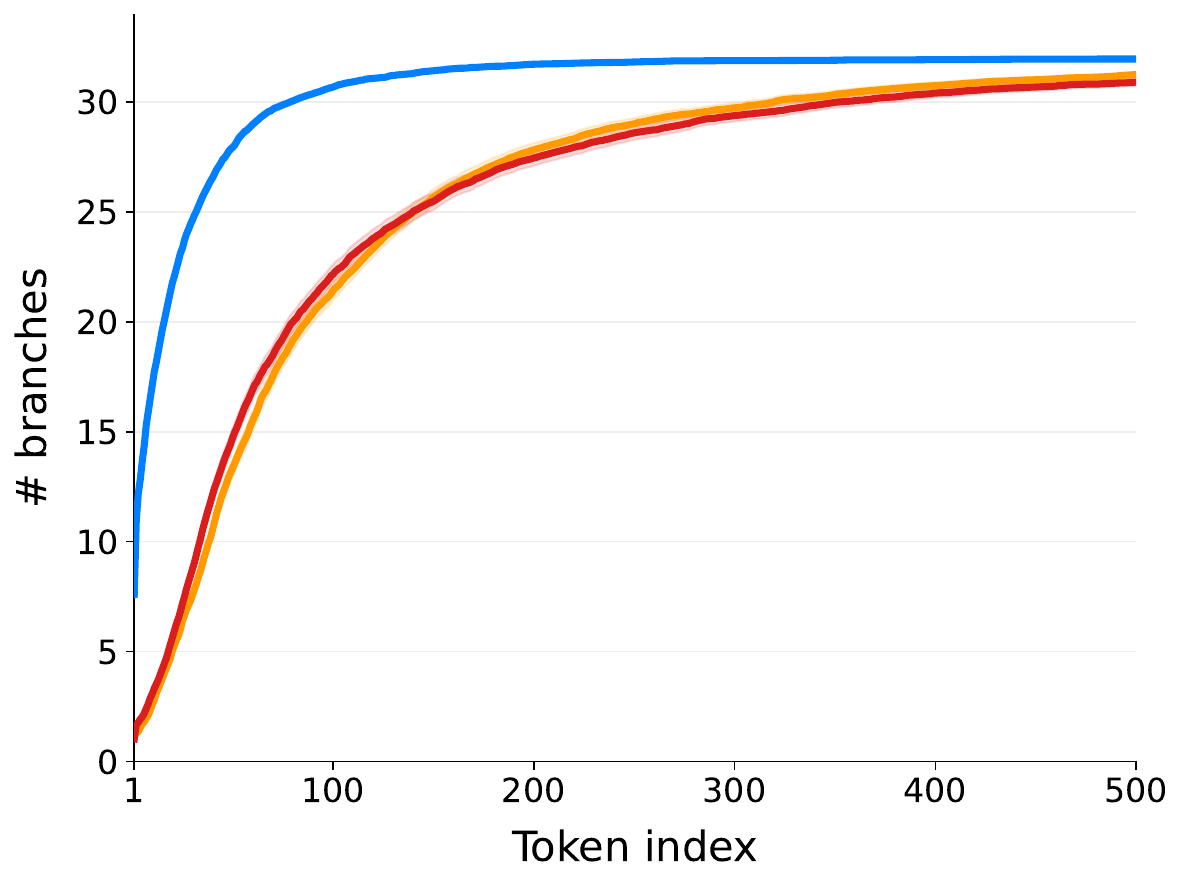}
        \caption{Minerva}
        \label{fig:branches-minerva}
    \end{subfigure}

    \vspace{0.6em}

    \includegraphics[width=0.6\textwidth]{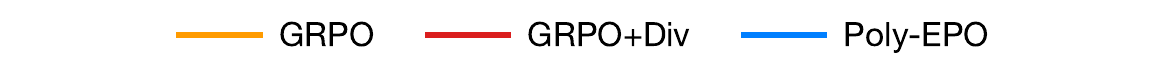}

    \caption{
    \textbf{\ours{} promotes broader branching during reasoning generation.}
    \textbf{Top:} Branching structure of 8 sampled rollouts on an AIME 2026 problem for \ours{} (left) and \grpo{} (right). Nodes represent shared-prefix clusters, and edges denote branching into distinct continuations.
    \textbf{Bottom:} Average number of active branches as a function of token position on AIME 2026, BeyondAIME, and Minerva. \ours{} consistently maintains more branches earlier in generation, indicating broader exploration over reasoning trajectories.
    }
    \label{fig:rollout_and_branching_behavior}
\end{figure*}

We next measure the fraction of prompts in each training batch for which the model samples at least one correct generation. As shown in \cref{fig: training_dynamics_maths_reasoning} (right), methods that explicitly encourage exploration achieve higher coverage: both \divrl{} and \ours{} outperform \grpo{} for much of training. This suggests that promoting exploration not only preserves strategic diversity, but also improves the probability of uncovering successful solutions during training. These two results, jointly, allow us to see that the polychromic objective does indeed enable the model to synergistically explore and exploit. 

\paragraph{Branching in Generations.}We further analyze the distinct behaviors induced by each post-training method by studying the diversity of model generations on the test set after fine-tuning with \grpo{} or \ours{}. In this analysis, we characterize diversity through the \textit{branching structure} of sampled generations. For each prompt \(x\) in the test set, we sample \(k\) generations \(y_1,\dots,y_k \sim \pi_\theta(\cdot \mid x)\) and organize them into a tree according to their shared token prefixes. Specifically, if two generations agree on their prefix through token \(h-1\) but differ at token \(h\), we say that they branch at token \(h\). Since all generations begin with the same beginning-of-sequence token (\texttt{[BOS]}), the tree has a single root.

In \cref{fig:rollout_and_branching_behavior}, we visualize example rollout trees and also quantify branching by measuring, at each token position, the number of active branches among sampled generations. We find that models trained with \ours{} branch substantially earlier in the generation process than models trained with \grpo{}. Although the \grpo{}-trained model eventually reaches a similar number of branches at later token positions, its generations remain coupled for much longer before diverging. This suggests that \ours{} encourages the model to differentiate its reasoning trajectories earlier, whereas \grpo{} tends to delay diversification until later in the rollout. Because early branching more naturally corresponds to committing to distinct reasoning strategies, these results indicate that \ours{} promotes broader strategic exploration at test time.

\paragraph{Majority Voting Evaluations.}Finally, we analyze whether \ours{} benefits from increased test-time compute. We evaluate using majority voting~\citep{wang2022self}, where \(k\) generations are sampled and the most frequent final answer is selected. The results are shown in \cref{fig:maj_baime_aime26_aime25}. We first observe that models trained with \ours{} retain substantially greater diversity at inference time: across benchmarks, the majority vote share under \ours{} is consistently lower than under both \grpo{} and \divrl{}, indicating that probability mass is distributed across a broader set of candidate answers rather than collapsing onto a single response.

Despite this higher entropy over final answers, \ours{} attains equal or stronger majority-vote accuracy as \(k\) increases, often yielding the largest gains from additional samples. This suggests that the extra diversity induced by the polychromic objective is not merely random variation, but reflects useful diversity in reasoning trajectories. We conjecture that exploration during post-training improves calibration and generalization, allowing additional test-time samples to be converted more effectively into performance gains.

\begin{figure}
    \centering    \includegraphics[width=1.0\linewidth]{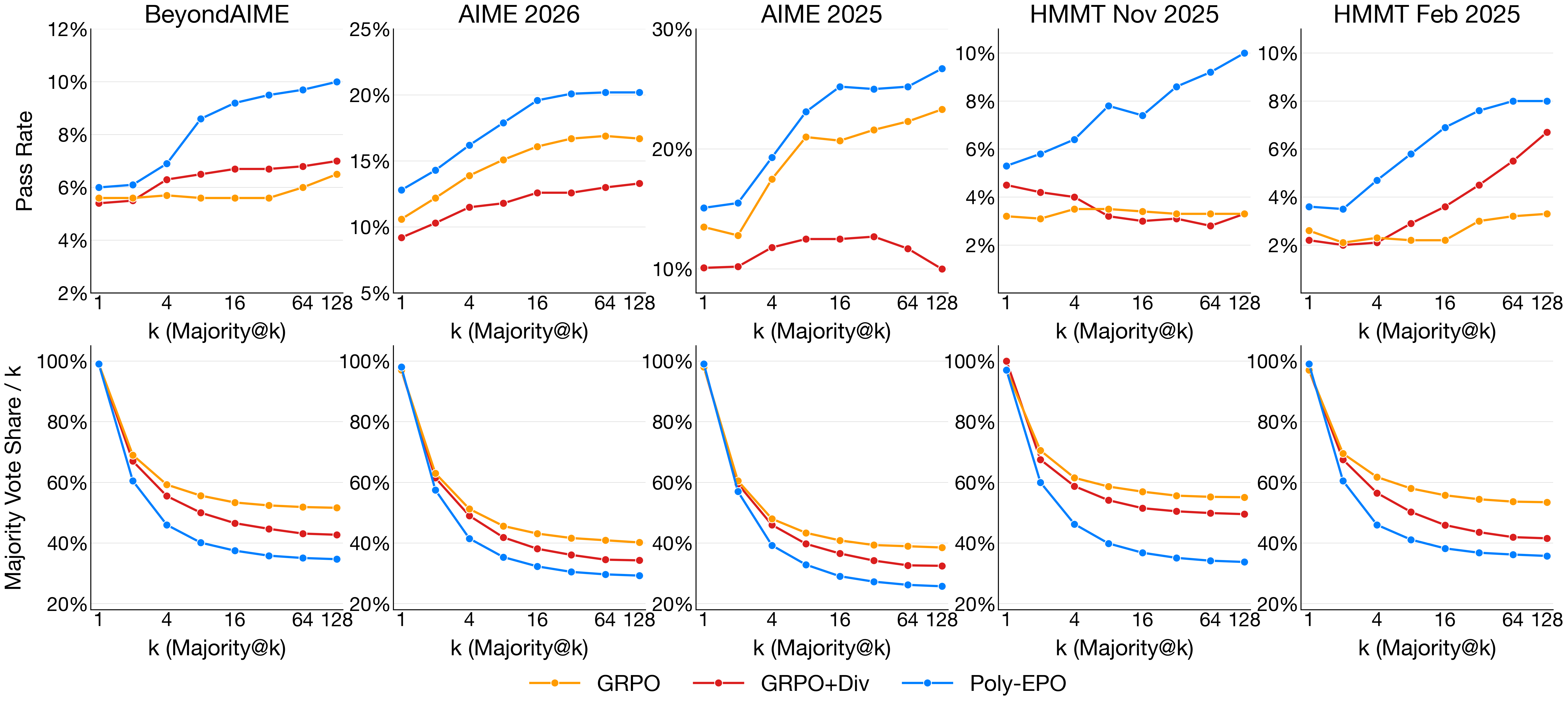}
    \caption{\textbf{Majority@$k$ evaluations on Test Sets}. The x-axis denotes the number of sampled generations $k$ used at test time. Top row: pass rate after selecting the answer returned by majority voting over the $k$ samples. Bottom row: majority vote share, i.e., the fraction of the $k$ votes assigned to the winning answer. Higher values indicate stronger agreement among sampled generations, while lower values indicate greater response diversity. All methods are initialized from \texttt{Qwen3-4B-Base}. The LM-judge used by \ours{} and \divrl{} for clustering responses is \texttt{Qwen-3-4B-Instruct}.}
    \label{fig:maj_baime_aime26_aime25}
\end{figure}

\subsection{Synthetic Domains with Infinitely Many Valid Strategies}
\label{subsec: experiments_infinite_strategy}

We now test our method in a controlled setting where each task admits infinitely many distinct strategies that can yield a correct answer. Our goal is to test whether \ours{} discovers a broader repertoire of successful strategies than \grpo{}. This probes a central property of exploration: an effective learner should not only achieve reward, but also discover diverse ways of achieving it through exploration. We consider two tasks in this section:

\begin{itemize}
    \item \textbf{Polynomial Solving.} In this task, the model is given a polynomial relation and must output a pair $(x,y)\in\mathbb{R}^2$ that satisfies it. For example, given $y = x^2 + 3x + 7$, one valid solution is $(x=1, y=11)$. 
    \item \textbf{Multi-digit Multiplication.} In this task, the model is asked to compute the product of two integers by decomposing the operands and distributing the multiplication. For example, to compute $343 \times 67$, one valid decomposition is
    \( (340 + 3)(70 - 3),\) which expands to
    \((340 \times 70) - (340 \times 3) + (3 \times 70) - (3 \times 3) = 22981.\)
\end{itemize}

For both tasks, we use \texttt{Qwen-3-1.7B-Base} as the base model and compare between \grpo{} and \ours{}. We use \texttt{Gemini-2.0-Flash} as the LM-judge for clustering generations when measuring diversity. For multi-digit multiplication, the judge clusters responses according to the strategy used; for example, a solution based on factorization is assigned to a different cluster than one based on partial products over place values. For polynomial solving, diversity is defined directly over the final answer: two outputs $(x_1,y_1)$ and $(x_2,y_2)$ are assigned to different clusters whenever $x_1 \neq x_2$ or $y_1 \neq y_2$.

\cref{fig:toy_exp_training} shows the training dynamics of the models, including how the diversity of generations evolves over training. In particular, we measure the number of distinct clusters among correct responses for each prompt, averaged over prompts in a batch, to quantify the diversity of successful strategies learned by the policy. We also measure the number of distinct clusters among incorrect responses to quantify the diversity of incorrect but exploratory generations produced by the policy.

We find that \grpo{} quickly collapses to a single successful strategy. In contrast, \ours{} continues to explore and discovers more than five times as many distinct successful strategies. Furthermore, \ours{} preserves higher diversity over incorrect generations, suggesting that when exploring optimistically, it has a larger coverage over strategies that are being attempted. These results indicate that \ours{} is substantially more effective at maintaining exploration and uncovering novel high-reward strategies in domains with many valid solutions.

\begin{figure}[H]
\centering

\makebox[\textwidth]{%
\begin{minipage}{0.46\textwidth}
    \centering
    \begin{subfigure}{0.49\linewidth}
        \centering
        \includegraphics[width=\linewidth]{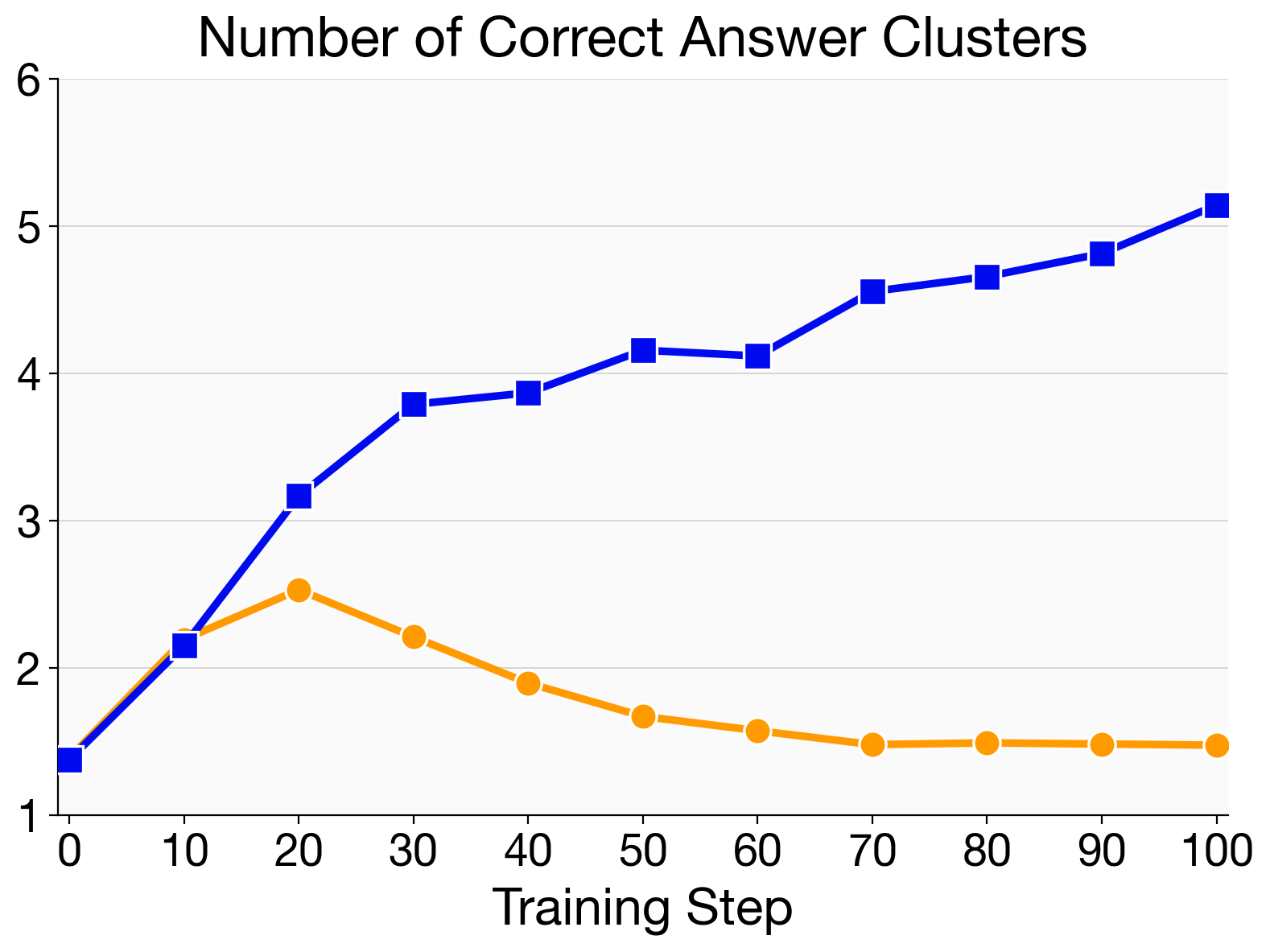}
    \end{subfigure}
    \hfill
    \begin{subfigure}{0.49\linewidth}
        \centering
        \includegraphics[width=\linewidth]{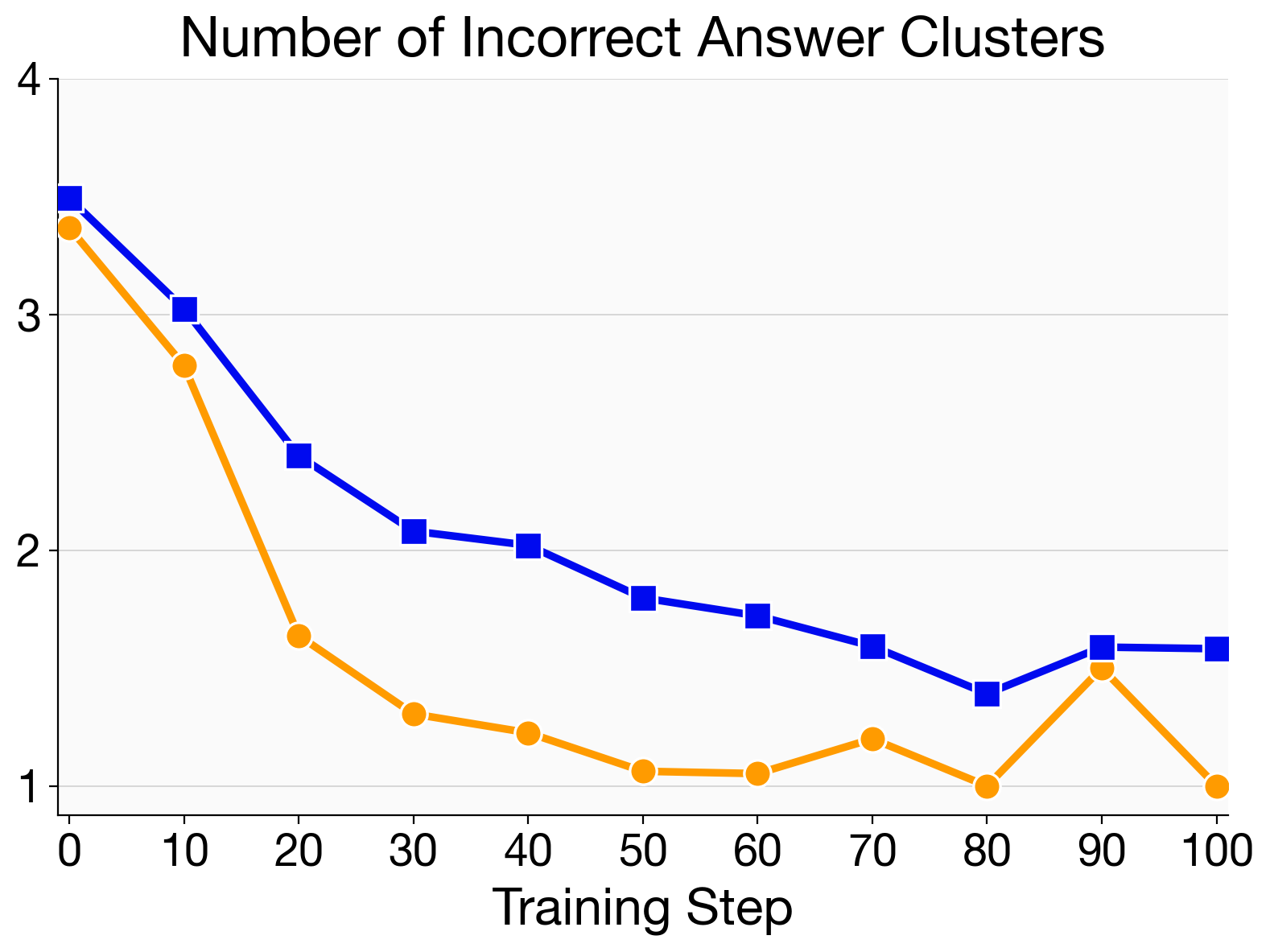}
    \end{subfigure}

    \vspace{0.7em}

    {\small Multi-digit Multiplication}
\end{minipage}
\hspace{0.02\textwidth}
\vrule width 0.7pt
\hspace{0.02\textwidth}
\begin{minipage}{0.46\textwidth}
    \centering
    \begin{subfigure}{0.49\linewidth}
        \centering
        \includegraphics[width=\linewidth]{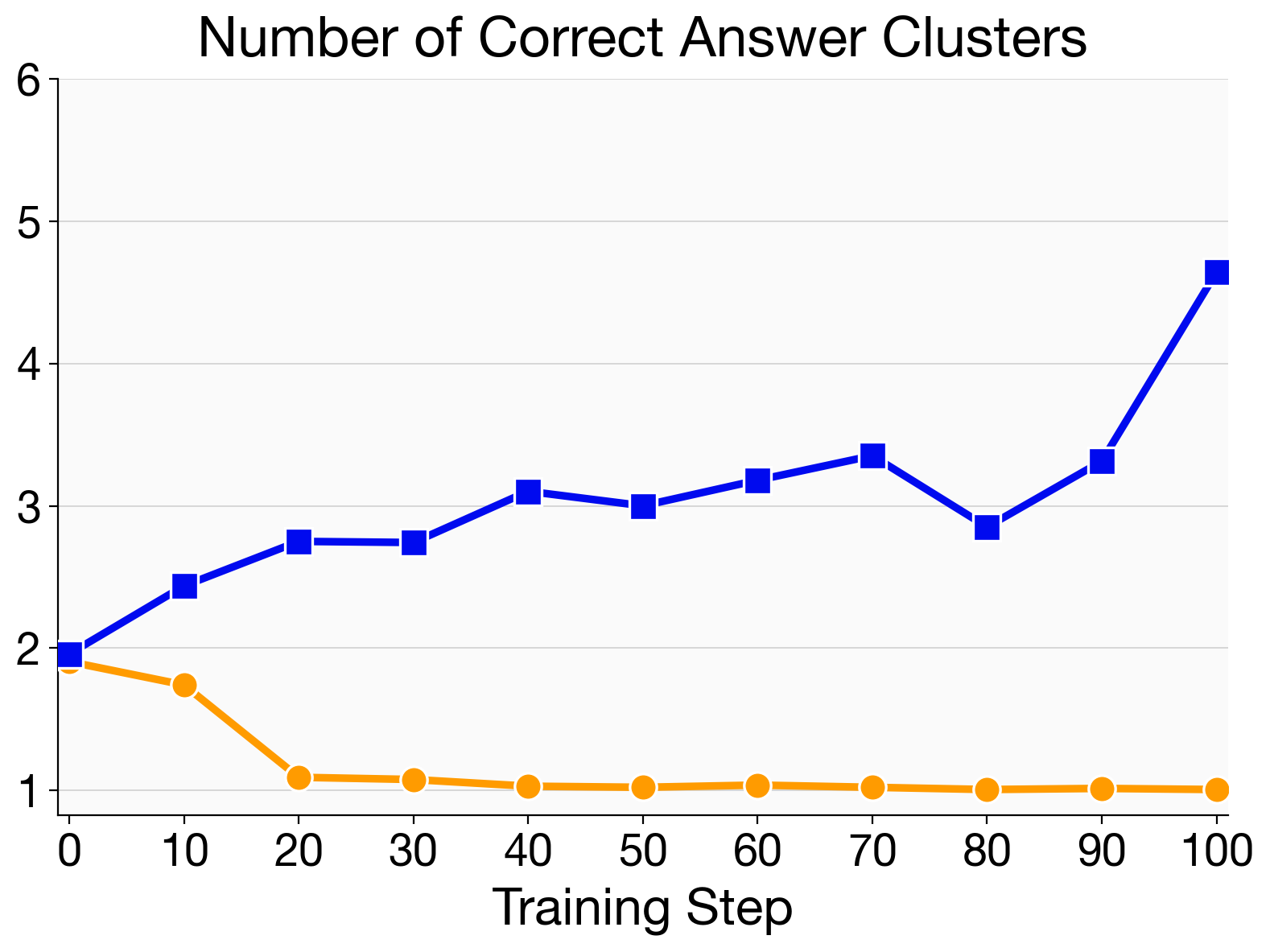}
    \end{subfigure}
    \hfill
    \begin{subfigure}{0.49\linewidth}
        \centering
        \includegraphics[width=\linewidth]{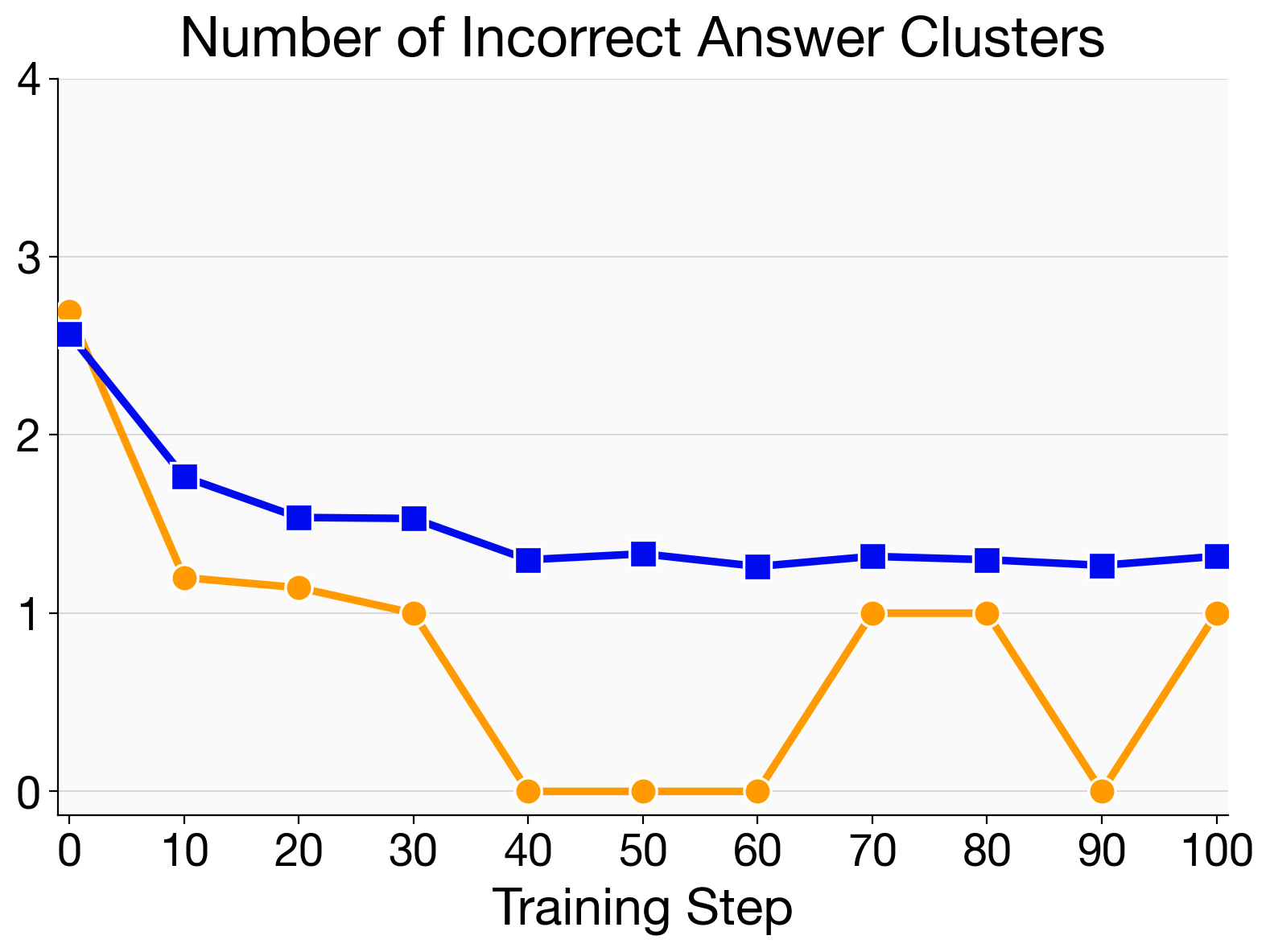}
    \end{subfigure}

    \vspace{0.7em}

    {\small Polynomial Solving}
\end{minipage}
}

\vspace{1em}

\includegraphics[width=0.40\textwidth]{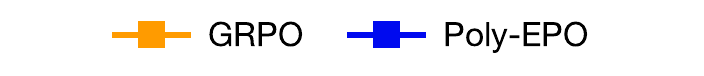}

\caption{
\textbf{Training dynamics on synthetic domains}. The two left plots correspond to \textbf{multi-digit multiplication}, and the two right plots correspond to \textbf{polynomial solving}. We measure strategy diversity by clustering generations and reporting the number of distinct clusters among correct and incorrect responses for each prompt, averaged across the batch. Relative to \grpo{}, \ours{} maintains substantially greater diversity throughout training and discovers many more successful strategies.
}
\label{fig:toy_exp_training}
\end{figure}

\section{Related Work}

\paragraph{Policy Gradient Methods.} Policy gradient methods \citep{10.5555/3009657.3009806, 10.5555/3312046, NIPS2001_4b86abe4} are a foundational class of reinforcement learning algorithms that directly optimize a policy to maximize expected rewards. Recent advances in reducing variance and improving sample efficiency \citep{BHATNAGAR20092471, degris2013offpolicyactorcritic, DBLP:journals/corr/LillicrapHPHETS15, wang2017sampleefficientactorcriticexperience, schulman2017trustregionpolicyoptimization, schulman2017proximalpolicyoptimizationalgorithms} have made them widely adopted for language model (LM) fine-tuning. In the LM setting, these methods typically omit a learned critic in favor of empirical advantage estimates \citep{Guo_2025, yu2025dapoopensourcellmreinforcement, zheng2025groupsequencepolicyoptimization, liu2025understandingr1zeroliketrainingcritical, minimax2025minimaxm1scalingtesttimecompute}. Concurrently, \cite{hamid2026polychromicobjectivesreinforcementlearning} introduced the set reinforcement learning framework, which optimizes an agent with respect to a set-level learning signal. Our work bridges and complements both lines of research: the general set RL recipe we propose can seamlessly adapt standard RL algorithms by replacing the advantage with the marginal set advantage to implement set RL at scale.

Furthermore, as test-time compute has become a standard paradigm for performance gains, recent research has explored training objectives aligned with inference-time metrics \citep{tang2025optimizinglanguagemodelsinference, chow2025inferenceawarefinetuningbestofnsampling, chen2025passktrainingadaptivelybalancing, walder2025passkpolicyoptimizationsolving}. While some of these frameworks also optimize over a number of generations for the specific instance of pass@$k$ training \citep{walder2025passkpolicyoptimizationsolving, tang2025optimizinglanguagemodelsinference}, they rely on a leave-one-out baseline. This causes the objective to collapse to individualized credit assignment that set RL, by definition, does not allow~\citep{hamid2026polychromicobjectivesreinforcementlearning}. Crucially, if one were to strictly optimize the pass@$k$ objective using the set RL framework, generations with incorrect final answers would always receive a negative advantage, as we showed in \S\ref{subsec: analysis_of_general_recipe}. On the other hand, \cite{chen2025passktrainingadaptivelybalancing} propose an algorithm for pass@$k$ training that also computes a U-statistic estimator. However, their advantage is not the same as the set advantage since they divide by the standard deviation in the set scores, which induces a significant deviation from our general set RL recipe. Additionally, they analytically compute the advantage of each generation under the pass@$k$ training framework, which may not be possible for arbitrary objectives. 

\paragraph{Exploration and Diversity.} Balancing exploration and exploitation remains a critical challenge during LM post-training. While some works argue that RL fine-tuning acts primarily as an entropy-reducing mechanism that does not surface new capabilities \citep{yue2025doesreinforcementlearningreally, cui2025entropymechanismreinforcementlearning, wu2026invisibleleashrlvrescape, zhao2025echochamberrlposttraining}, others have demonstrated that RL fine-tuning can successfully expand reasoning capabilities to held-out tasks \citep{liu2025prorlprolongedreinforcementlearning, research2026composer2technicalreport}. Our work is simply motivated by the principle that effective exploration amplifies learning from experience. Several modifications to the RL pipeline have been proposed to achieve such exploration. Standard approaches like entropy bonuses \citep{haarnoja2017reinforcementlearningdeepenergybased, haarnoja2018softactorcriticoffpolicymaximum, schulman2017proximalpolicyoptimizationalgorithms, seo2021stateentropymaximizationrandom, islam2019marginalizedstatedistributionentropy, zheng2025returnentropyelicitingexplore} encourage local exploration, though their efficacy in inducing trajectory-level exploration over high-level strategies during LM post-training remains unclear. Similarly, UCB and count-based bonuses \citep{song2025outcomebasedexplorationllmreasoning, zhang2025countcountsmotivatingexploration, tuyls2025representationbasedexplorationlanguagemodels} can incentivize novel behaviors, but they rely on weighting hyperparameters rather than explicitly driving a synergy between exploration and exploitation. Other methods implicitly induce exploration through curiosity-driven techniques \citep{dai2025cdecuriositydrivenexplorationefficient, gao2026navigateunknownenhancingllm} or novel objectives independent of diversity measures \citep{tajwar2026maximumlikelihoodreinforcementlearning, chen2025passktrainingadaptivelybalancing, walder2025passkpolicyoptimizationsolving, cui2025entropymechanismreinforcementlearning, yu2025dapoopensourcellmreinforcement, gxchen2025klregularizedreinforcementlearningdesigned}.

More closely related to our approach are works that promote exploration via objectives targeting the semantic diversity of generations \citep{li2025jointlyreinforcingdiversityquality, yao2025diversityawarepolicyoptimizationlarge, hu2026rewardingrareuniquenessawarerl}. However, while these distribute probability mass across multiple \textit{correct} generations, they do not explicitly encourage optimistic exploration by providing a positive learning signal to exploratory trajectories with incorrect final answers. On the other hand, \cite{song2025outcomebasedexplorationllmreasoning, tuyls2025representationbasedexplorationlanguagemodels} use the standard RL framework with an exploration bonus that can encourage such optimistic exploration but the balance between exploration and exploitation is dictated by a weighting hyperparameter and the objective does not lead to an advantage function that explicitly encourages a synergy between these goals. In contrast, the set reinforcement learning framework with polychromic objectives \citep{hamid2026polychromicobjectivesreinforcementlearning} achieves both optimistic exploration and synergistic exploitation. Yet, their algorithm requires a learned critic and relies on computationally heavy vine sampling, making it difficult to scale for LM post-training. By resolving these sample and computational complexities, our approach allows LMs to leverage the full benefits of polychromic objectives.

\section{Conclusion}
In this paper, we presented \oursfull{} (\ours{}), a set reinforcement learning algorithm that trains a policy to explore diverse reasoning strategies optimistically while balancing exploration and exploitation through an explicit synergy between them. The optimistic nature of exploration is reflected in the fact that the algorithm can upweight generations that pursue markedly different strategies even when they do not produce correct answers. This synergy is encoded directly in the advantage function, which depends on the covariance between average reward and diversity across generations in a set. Our results suggest that \ours{} enables the model to explore diverse reasoning pathways that enables it to unlock performance at test-time especially when given multiple attempts per problem. 
We close by noting several limitations of this work. First, as with many methods that use LM-judges to provide learning signals in reinforcement learning, our approach may be susceptible to reward hacking when an LM-judge is used to cluster responses according to high-level strategy. Second, if response lengths grow substantially, clustering all generations for a prompt in a single judge call may place significant demands on long-context reasoning, which remains challenging for many LM-judges. Third, the scaling laws of our general recipe for set reinforcement learning are still not well understood, especially regarding the relative roles of set size, number of rollouts, and number of constructed sets. Future work should study these questions more systematically. Furthermore, alternative forms of measuring diversity in generations, including ones that do not incorporate an LM-judge should be studied. Future works should also compare this approach to alternative ways of encouraging exploration, such as training on synthetic data or more implicit training objectives.

\section{Acknowledgments}

We thank Michael Li, Suvir Mirchandani, Jensen Gao, Omar Shaikh, and Joey Hejna for their feedback throughout the project and reviews on early drafts of the paper; these comments were crucial in developing and finalizing both the ideas and the implementations in the project. We thank Fahim Tajwar for open-sourcing the excellent codebase used in \cite{tajwar2026maximumlikelihoodreinforcementlearning} which is what we used for this project and for giving us helpful advice on the experiments we conducted. We also thank Yoonho Lee for providing valuable insights, especially on the practical use cases and limitations of search methods. JIH would like to dedicate this work to Juno---a great friend, and a source of endless inspiration and happiness to his life. 

This work is supported by Schmidt Sciences; ONR grants N00014-22-1-2621 and N00014-22-1-2293; NSF projects 1941722 and 2125511; an award from DARPA ExpMath; and Google DeepMind, which provided Google Cloud Platform credits.

\printbibliography
\newpage

\appendix

\section{Implementation Details}
\label{appendix: implementation-details}

In this section, we discuss the implementation details of \ours{}. Our implementation builds upon the codebase from \cite{tajwar2026maximumlikelihoodreinforcementlearning}, which is in turn based on the Verl framework \citep{sheng2024hybridflow}. Recall that our general recipe for set RL simply adapts the advantage function used in a standard RL algorithm. As our backbone, we use the following standard RL objective~\citep{schulman2017proximalpolicyoptimizationalgorithms, shao2024deepseekmathpushinglimitsmathematical, Guo_2025, liu2025understandingr1zeroliketrainingcritical}:

$$J(\theta) = \mathbb{E}_{x \sim \mathcal{D}, y_{1:N} \sim \pi_{\theta_\mathrm{old}}(\cdot \mid x)}\left[ \frac{1}{N}\sum_{i=1}^N \frac{1}{T_i} \sum_{t=1}^{|y_i|} \min \left[\omega_{i,t}\hat{A}_{i,t}, \mathrm{clip}\left( \omega_{i,t}, 1 - \epsilon_\mathrm{low} , 1 + \epsilon_\mathrm{high}\right)\hat{A}_{i,t} \right]\right],$$

where $\omega_{i,t} = \frac{\pi_\theta(y_{i,t} \mid x, y_{i,<t})}{\pi_{\theta_\mathrm{old}}(y_{i,t} \mid x, y_{i,<t})}$. 

For the \grpo{} baseline, $T_i = |y_i|$ to ensure the objective is averaged over the length of each individual generation. The advantage estimator is defined as:
$$\hat{A}_{i, t} = r(x, y_i) - \mathrm{mean}(\{r(x, y_1), \dots, r(x,y_N)\}).$$
Note that we omit the standard deviation normalization term originally used in \cite{shao2024deepseekmathpushinglimitsmathematical}, as we found this choice led to better empirical performance—a modification consistent with recent findings in \textsc{Dr.GRPO}~\citep{liu2025understandingr1zeroliketrainingcritical} and other works, such as \cite{research2026composer2technicalreport}.

For the \divrl{} baseline, the advantage estimator incorporates a diversity penalty and is computed as:
$$\hat{A}_{i,t} = r(x, y_i) + \lambda d(x, y_i) - \mathrm{mean}\left(\{ r(x, y_j) + \lambda d(x, y_j) \}_{j=1}^N \right).$$

Finally, for \ours{}, we do not normalize by the individual generation lengths. Instead, following the \textsc{Dr.GRPO} implementation in Verl, we set $T_i = T_\mathrm{max}$, where $T_\mathrm{max}$ is the maximum response length. The advantage is computed as:
$$\hat{A}_{i,t} = \widehat{A^\sharp_\mathrm{marg}}(x, y_i)$$
as detailed in \cref{sec: set_rl_for_LMs}. Note that although we proved, in Proposition \ref{prop:set_rl_u_stat_unbiased}, that our estimator is unbiased only after scaling by a constant factor that can be absorbed into the learning rate, we do not introduce this factor in implementation and do not consider this to be an additional hyperparameter that needs tuning.

\subsection{LM-Judge and Diversity Function}
\label{appendix: lm-judge_and_diversity_function}

For a given prompt $x$, we sample $N$ candidate responses $y_1,\dots,y_N \overset{\mathrm{i.i.d}}{\sim} \pi_\theta(\cdot \mid x)$ and use a language model as judge to cluster these generations. The primary objective of this clustering phase is to group responses strictly according to the underlying mathematical reasoning strategies employed, rather than surface-level textual similarities. To systematically guide the judge, we provide a comprehensive system prompt comprising task specifications, rigorous clustering criteria, formatting constraints, and curated in-context examples. The core directives are defined as follows:

\begin{enumerate}
    \item \textbf{Strategy-centric clustering:} The judge must categorize responses based on both macro-strategy (the overarching conceptual framework) and micro-strategy (the specific techniques used at key intermediate steps), deliberately disregarding stylistic variations or tone. Crucially, this clustering process is completely independent of the final derived answer. If two response trajectories exhibit identical logical reasoning but diverge due to a minor arithmetic error, they must be assigned to the same cluster. The provided few-shot examples explicitly demonstrate this invariant.
    \item \textbf{Isolation of degenerate responses (Cluster 100):} We reserve a dedicated, out-of-bounds index (``cluster 100'') for degenerate generations. Empirically, we observe that during optimization, the policy network may sometimes discover reward-hacking behaviors, such as directly predicting the final answer without executing intermediate logical steps or generating unintelligible text. The judge is explicitly instructed to isolate all such spurious generations into this dedicated ``cluster 100'' bucket. Because the number of generations per prompt, $N$, is strictly less than 100, this fixed numeric identifier prevents index collisions and facilitates the systematic tracking of undesirable model behaviors throughout training.
\end{enumerate}

The full prompt supplied to the LM-judge consists of: (i) a static instruction block detailing the task, rules, and few-shot examples, and (ii) an instance-specific suffix containing the problem context alongside the sampled responses.

\begin{tcolorbox}[
    colback=gray!5,
    colframe=gray!60,
    title=Instruction Block for LM-Based Clustering of Reasoning Strategies,
    breakable
]
{\footnotesize
\begin{lstlisting}[style=promptstyle]
Your ONLY task is to cluster the {n_responses} responses into buckets based on their reasoning algorithm, including both the overall strategy and the methods used at key intermediate steps.

**INPUT FORMAT:** You will receive:
1) A "Context" describing the task.
2) A numbered list of Responses from 1 to {n_responses}. Each response contains a reasoning process and final answer. 
Note: Responses may or may not explicitly state their strategy; you must infer the strategy by analyzing the mathematical steps taken.

**CLUSTERING CRITERIA:**
(1) Macro-strategy: The overall conceptual framework (e.g., recursion vs infinite series; prime factorization vs gcd-based formula).
(2) Micro-strategy: The specific method used to resolve key intermediate steps. Examples include: how absolute values are removed (+- case split vs squaring), how intervals are partitioned, or how a basis is chosen.

**CLUSTERING RULES:** - Cluster strictly based on logic and approach. NOT on wording, tone, formatting, or final answer.  
- Two responses share a cluster_id IF AND ONLY IF they use the same macro-strategy AND the same micro-strategy at every key step.
- Arithmetic errors do NOT create new clusters if the underlying logic is identical.
- **SPECIAL CLUSTER 100:** You MUST assign `cluster_id: 100` to any response that is:
    * Gibberish (random characters, nonsense strings).
    * Irrelevant to the math problem (off-topic text).
    * Non-mathematical reasoning (e.g., writing code to solve it instead of math, or making a random guess at the final answer without logical steps).

**OUTPUT RULES (STRICT):** 1. Respond ONLY with a JSON object. No text outside the JSON.
2. The JSON must contain exactly {n_responses} keys: "1", "2", ..., "{n_responses}".  
3. The value for each key must be:  
    "chain_of_thought": "Macro: [short description]. Micro: [short description]."
    "cluster_id": integer.
4. chain_of_thought must be concise and avoid repeating the actual calculations.

**Few-Shot Example 1:**

**Context:**
What is the smallest value of x such that |5x - 1| = |3x + 2|?

**Responses:**
1. We can split this into two cases: 5x - 1 = 3x + 2 or 5x - 1 = -(3x + 2). Solving the first gives 2x = 3 so x = 1.5. The second gives 8x = -1 so x = -0.125.
2. The expression 5x-1 changes sign at 1/5, and 3x+2 changes at -2/3. For x < -2/3, we have -(5x-1) = -(3x+2). For -2/3 < x < 1/5, we have -(5x-1) = 3x+2. Solve -(5x - 1) = 3x + 2 for the range, yielding x = -0.125.
3. Using the property that |a|=|b| implies a=b or a=-b, we get 5x-1 = 3x+2 (x=1.5) and 5x-1 = -3x-2 (x=-0.125). So the answer is x = -0.125
4. To get rid of the absolute values, square both sides: (5x - 1)^2 = (3x + 2)^2. This expands to 25x^2 - 10x + 1 = 9x^2 + 12x + 4. Solve 16x^2 - 22x - 3 = 0. So, x = -1/8, 3/2. Final answer is x = -1/8
5. Either 5x - 1 = 3x + 2 or 5x - 1 = -3x - 2. This leads to x = 3/2 and x = 0. So, final answer is x = 0.
6. I think the answer is probably 0 or maybe 1.5. 
7. asdf qwer zxcv 9999 ---- ??? Let's write Python to check each x from -10 to 10: `if abs(5*x-1) == abs(3*x+2): print(x)`. The answer is -0.125. 

**Expected Output:**
{{
"1": {{"chain_of_thought": "Macro: Algebraic casework. Micro: Direct +- case split to remove absolute values.", "cluster_id": 1}},
"2": {{"chain_of_thought": "Macro: Interval analysis. Micro: Testing expression sign changes across number line partitions.", "cluster_id": 2}},
"3": {{"chain_of_thought": "Macro: Algebraic casework. Micro: Direct +- case split to remove absolute values.", "cluster_id": 1}},
"4": {{"chain_of_thought": "Macro: Algebraic transformation. Micro: Squaring both sides to create and solve a quadratic equation.", "cluster_id": 3}},
"5": {{"chain_of_thought": "Macro: Algebraic casework. Micro: Direct +- case split to remove absolute values (contains arithmetic error).", "cluster_id": 1}},
"6": {{"chain_of_thought": "Macro: Non-mathematical. Micro: Random guessing without any logical derivation.", "cluster_id": 100}},
"7": {{"chain_of_thought": "Macro: Gibberish/Non-mathematical. Micro: Contains random non-English text, nonsense strings, and code-based iteration.", "cluster_id": 100}}
}}

**Few-Shot Example 2:**

**Context:**
What is the least common multiple of 72 and 96?

**Responses:**
1. 72 = 2^3 * 3^2. 96 = 2^5 * 3^1. To find the LCM, we take the highest power of each prime factor present: 2^5 * 3^2 = 32 * 9 = 288.
2. Prime factors of 72: 2, 2, 2, 3, 3. Prime factors of 96: 2, 2, 2, 2, 2, 3. The union of these sets is five 2s and two 3s. Total: 276.
3. First find the GCD using the Euclidean algorithm: 96 = 72(1) + 24; 72 = 24(3) + 0. GCD is 24. LCM is (72 * 96) / 24.
4. 72 = 8*9, 96 = 8*12. The answer is 288. The answer is 288. The answer is 288. The answer is 288.

**Expected Output:**
{{
"1": {{"chain_of_thought": "Macro: Prime factorization analysis. Micro: LCM via maximum exponents of prime factors.", "cluster_id": 1}},
"2": {{"chain_of_thought": "Macro: Prime factorization analysis. Micro: LCM via maximum exponents of prime factors (contains arithmetic error).", "cluster_id": 1}},
"3": {{"chain_of_thought": "Macro: Product-GCD relationship. Micro: GCD calculation via Euclidean algorithm followed by the LCM formula.", "cluster_id": 2}},
"4": {{"chain_of_thought": "Macro: Excessive repetition. Micro: Response loops the final answer multiple times at the end.", "cluster_id": 100}}
}}
\end{lstlisting}
}
\end{tcolorbox}

\vspace{0.5em}
\noindent\textbf{Instance-Specific Suffix.}
For a given problem instance, we dynamically construct the prompt suffix by populating the template below with the actual mathematical context and the sampled responses:

\begin{tcolorbox}[
    colback=gray!5,
    colframe=gray!60,
    title=Instance-Specific Prompt Suffix,
    breakable
]
{\footnotesize
\begin{lstlisting}[style=promptstyle]
**Context:**
<problem description>

**Responses:**
1. <response 1>
2. <response 2>
...
{n_responses}. <response {n_responses}>
\end{lstlisting}
}
\end{tcolorbox}

During the evaluation phase, the static instruction block and the dynamically instantiated suffix are concatenated into a unified user message and passed to the LM-judge for inference. The LM-judge is used to cluster all $N$ responses, $y_1,\cdots,y_N$, sampled conditioned on a prompt. This gives us a cluster assignment, $\mathcal{C}(y)$, for each generation $y$. In order to prevent the model from artificially inducing diversity by generating random guesses at the final answer or by generating unintelligible text, we remove the cluster assignments for all such generations when computing the diversity of any set. In other words, when computing the diversity of any set using \cref{eq: diversity_function}, we remove any cluster assignment to ``cluster 100'' from being in the set in the numerator.

\subsection{Hyperparameters}

The training hyperparameters for the experiments on mathematical reasoning (\S\ref{subsec: experiments_maths_reasoning}) are provided in \cref{tab: training_hyperparameters}. 

\begin{table}[H]
\centering
\begin{tabular}{ll}
\toprule
\textbf{Parameter} & \textbf{Value} \\
\midrule
Base model & \texttt{Qwen-3-4b-base} \\
Generations per prompt & 8 \\
Set size (for set RL) & 4 \\
Number of sets (for set RL) & 70 \\
Max prompt length & 1024 \\
Learning rate & $1 \times 10^{-6}$ \\
KL coefficient & 0.0 \\
Clip ratio $\epsilon_\mathrm{low}$ & 0.20 \\
Clip ratio $\epsilon_\mathrm{high}$ & 0.28 \\
Entropy coefficient & 0.0 \\
Rollout temperature & 1.0 \\
Prompts per batch & 128 \\
Prompts per minibatch & 64 \\
Max response length & 4096 \\
Training steps & 850 \\
Device & 4 $\times$ NVIDIA H200 \\
\bottomrule
\end{tabular}
\caption{Training hyperparameters.}
\label{tab: training_hyperparameters}
\end{table}

\newpage

\section{Proofs}
\label{appendix: proofs}

First, we prove Proposition \ref{prop:set_rl_u_stat_unbiased}. 

{\noindent \textbf{Proposition 3.1 (Restated)} Fix a prompt $x$, and let $y_1,\cdots,y_N \overset{\mathrm{i.i.d.}}{\sim} \pi_\theta(\cdot \mid x)$ be our independently sampled $N$ generations and let $f : \mathcal{X} \times \mathcal{Y}^{\oplus n} \rightarrow \mathbb{R}$ be our set objective. Then, $$\mathbb{E}[
\sum_{i=1}^N
\nabla_\theta \log \pi_\theta(y_i \mid x)\,
\widehat{A_{\mathrm{marg}}^\sharp}(x,y_i; f)]
=
M  \nabla_\theta
\mathbb{E}_{y_{1:n} \sim \pi_\theta(\cdot \mid x)}
\bigl[f(x,y_{1:n})\bigr],$$ where $M \in \mathbb{R}_{>0}$ is a constant (and depends on the number of sets we construct). Consequently, after also taking expectation over $x \sim \mathcal{D}$ and scaling the learning rate, the estimator is an unbiased estimator of the set RL gradient.}

\begin{proof}
We first prove the case where we construct all $K = \binom{N}{n}$ sets from $y_{1:N}$ in a combinatorial fashion. Let $G_{1:K}$ be all the unique sets we construct and let $\mathcal{G}(y)$ be the collection of sets that contain the generation $y$. Starting from our estimator, we get:
\begin{align*}
    & \mathbb{E}\left[ \sum_{i=1}^N \nabla_\theta \log \pi_\theta(y_i \mid x) A^\sharp_\mathrm{marg}(x, y_i) \right] \\
    & = \mathbb{E}\left[ \sum_{i=1}^N \nabla_\theta \log \pi_\theta(y_i \mid x) \left( \sum_{G \in \mathcal{G}(y_i)} (f(x, G) - \frac{1}{K} \sum_{j=1}^K f(x, G_j) \right) \right] \\
    & = \mathbb{E}\left[ \sum_{G \in G_{1:K}} (f(x, G) -  \frac{1}{K}\sum_{j=1}^K f(x, G_j)) \sum_{y_i \in G} \nabla_\theta \log \pi_\theta(y_i \mid x) \right] \\
    & = \mathbb{E}\left[ \sum_{G \in G_{1:K}} f(x, G) \sum_{y_i \in G} \nabla_\theta \log \pi_\theta(y_i \mid x) \right] - \mathbb{E}\left[ \sum_{G \in G_{1:K}} \frac{1}{K}\sum_{j=1}^K f(x, G_j) \sum_{y_i \in G} \nabla_\theta \log \pi_\theta(y_i \mid x) \right] \\
    & = K \mathbb{E}\left[ f(x, G) \sum_{y_i \in G} \nabla_\theta \log \pi_\theta(y_i \mid x) \right] - \frac{1}{K} \mathbb{E}\left[ \sum_{j=1}^K f(x, G_j) \sum_{G \in G_{1:K}}  \sum_{y_i \in G} \nabla_\theta \log \pi_\theta(y_i \mid x) \right] \\
    & = K \mathbb{E}\left[ f(x, G) \sum_{y_i \in G} \nabla_\theta \log \pi_\theta(y_i \mid x) \right] - \frac{1}{K} \mathbb{E}\left[ \sum_{j=1}^K f(x, G_j)  \cdot \binom{N-1}{n-1} \sum_{i = 1}^N  \nabla_\theta \log \pi_\theta(y_i \mid x) \right] \\
    & = K \mathbb{E}\left[ f(x, G) \sum_{y_i \in G} \nabla_\theta \log \pi_\theta(y_i \mid x) \right] - \frac{1}{K} \binom{N-1}{n-1}  \mathbb{E}\left[ \sum_{j=1}^K f(x, G_j)  \cdot \sum_{i = 1}^N  \nabla_\theta \log \pi_\theta(y_i \mid x) \right] \\
    & = K \mathbb{E}\left[ f(x, G) \sum_{y_i \in G} \nabla_\theta \log \pi_\theta(y_i \mid x) \right] - \frac{1}{K} \binom{N-1}{n-1}  \sum_{j=1}^K \sum_{i=1}^N \mathbb{E}\left[ f(x, G_j)  \cdot \nabla_\theta \log \pi_\theta(y_i \mid x) \right] \\
\end{align*}

Now we expand the second term. We can write the sum over tokens $y_1,\cdots,y_N$ as: 

{\allowdisplaybreaks \begin{align*}
    & \frac{1}{K} \binom{N-1}{n-1} \sum_{j=1}^K \sum_{i=1}^N \mathbb{E}\left[ f(x, G_j)  \cdot \nabla_\theta \log \pi_\theta(y_i \mid x) \right] \\ 
    & = \frac{1}{K} \binom{N-1}{n-1} \sum_{j=1}^K \left( \mathbb{E}\left[ f(x, G_j)  \cdot \sum_{y_i \in G_j} \nabla_\theta \log \pi_\theta(y_i \mid x) \right] + \sum_{y_i \not \in G_j} \mathbb{E}\left[ f(x, G_j)  \cdot \nabla_\theta \log \pi_\theta(y_i \mid x) \right]\right) \\ 
    & = \binom{N-1}{n-1}  \mathbb{E}\left[ f(x, G)  \cdot \sum_{y_i \in G} \nabla_\theta \log \pi_\theta(y_i \mid x) \right] - 0 
\end{align*}}

Combining, we get that: 
{\allowdisplaybreaks \begin{align*}
    \mathbb{E}\left[ \sum_{i=1}^N \nabla_\theta \log \pi_\theta(y_i \mid x) A^\sharp_\mathrm{marg}(x, y_i) \right] = \left( \binom{N}{n} - \binom{N-1}{n-1} \right) \mathbb{E}\left[ f(x, G) \sum_{y_i \in G} \nabla_\theta \log \pi_\theta(y_i \mid x) \right]. 
\end{align*} } By Pascal's rule, the scaling factor $M = \binom{N}{n} - \binom{N-1}{n-1}$ is greater than 0 as long as $N > n$. 

The proof for the case where we construct $K$ sets by uniformly sampling (without replacement) is provided in the next proposition.

\end{proof}

\begin{prop}
\label{prop: normalized_sampled_marginal_advantage_unbiased}
Let $\mathcal S := \{S \subseteq \{1,\cdots,N \} : |S| = n\}$ which gives us a collection of all possible sets of size $n$ without replacement (more precisely, it gives us a collection of indices which defines the collection of all possible sets). Let $K_\mathrm{all} := |\mathcal S| = \binom{N}{n}$ which is the maximum number of sets of size $n$ one can construct from the $N$ generations, and let $d := \binom{N-1}{n-1}$. Let $\mathcal T \subseteq \mathcal S$ be a uniformly sampled subset of size $K$, sampled without replacement and independently of $y_{1:N}$. 

Define $ q_K := \Pr_{\mathcal T}\left(C_i(\mathcal T)>0\right) = 1-\frac{\binom{\binom{N-1}{n}}{K}} {\binom{\binom{N}{n}}{K}},$ where we use the convention that $\binom{a}{K}=0$ when $K>a$. Then
$$
\mathbb E_{y_{1:N},\mathcal T}
\left[
\sum_{i=1}^N
\nabla_\theta \log \pi_\theta(y_i \mid x)\,
A^\sharp_{\mathrm{marg}}(x,y_i \mid \mathcal T)
\right]
=
M_K
\nabla_\theta
\mathbb E_{y_{1:n}\sim \pi_\theta(\cdot \mid x)}
\left[
f(x,y_{1:n})
\right],$$
where
$$ M_K = \frac{N}{n}q_K-1.$$
In particular, if $N>n$ and $K>1$, then $M_K>0$. In other words, up to the positive scalar factor $M_K$, we have an unbiased estimator of the set-RL gradient.
\end{prop}

\begin{proof}
For notational convenience, write $\nabla_i := \nabla_\theta \log \pi_\theta(y_i \mid x).$ We first decompose the left hand side as:
\begin{align*}
&\mathbb E_{y_{1:N},\mathcal T}
\left[
\sum_{i=1}^N
\nabla_i
A^\sharp_{\mathrm{marg}}(x,y_i \mid \mathcal T)
\right] \\
&=
\underbrace{
\mathbb E_{y_{1:N},\mathcal T}
\left[
\sum_{i=1}^N
\nabla_i
\mathbf 1\{C_i>0\}
\frac{1}{C_i}
\sum_{\substack{S\in \mathcal T\\ i\in S}}
f_S
\right]}_{\text{Term 1}}
-
\underbrace{
\mathbb E_{y_{1:N},\mathcal T}
\left[
\sum_{i=1}^N
\nabla_i
\mathbf 1\{C_i>0\}
\widehat f_{\mathcal T}(x)
\right]}_{\text{Term 2}}.
\end{align*}

We first simplify $\text{Term 1}$. 
\begin{align*}
\mathbb E_{y_{1:N},\mathcal T} \left[ \sum_{i=1}^N \nabla_i \mathbf 1\{C_i>0\} \frac{1}{C_i} \sum_{\substack{S\in \mathcal T\\ i\in S}} f_S \right] &= \mathbb E_{y_{1:N},\mathcal T}
\left[
\sum_{i=1}^N
\nabla_i \sum_{G \in \mathcal{S}, y_i \in G} \frac{\mathbf 1\{G \in \tau\}}{C_i}f(x, G) \right] \\
&= \mathbb E_{y_{1:N}}
\left[
\sum_{i=1}^N
\nabla_i \sum_{G \in \mathcal{S}, y_i \in G} f(x, G) \mathbb{E}_{\tau}\left[ \frac{\mathbf 1\{G \in \tau\}}{C_i} \right] \right] \\
\end{align*}

Now, note that \begin{align*}
    \mathbb{E}_{\tau}\left[ \sum_{G \in \mathcal{G}(y_i)} \frac{\mathbf1\{G \in \tau\}}{C_i} \right] = \sum_{G \in \mathcal{G}(y_i)}  \mathbb{E}_{\tau}\left[ \frac{\mathbf1\{G \in \tau\}}{C_i} \right]. 
\end{align*} On the other hand, \begin{align*}
    \mathbb{E}_{\tau}\left[ \sum_{G \in \mathcal{G}(y_i)} \frac{\mathbf1\{G \in \tau\}}{C_i} \right] = \mathbb{E}_{\tau}\left[ \mathbf{1} \{ C_i > 0\} \right] =: q_K, 
\end{align*} where $q_K$ is the probability that $y_i$ belongs to one of the $K$ sampled sets in $\tau$.  

Next, we note that because the sets are sampled uniformly and independently of $y_{1:N}$, we have that, for any $G$ and $G'$ in $\mathcal{S}$, $$\mathbb{E}_{\tau}\left[ \frac{\mathbf1\{G \in \tau\}}{C_i} \right] = \mathbb{E}_{\tau}\left[ \frac{\mathbf1\{G' \in \tau\}}{C_i} \right].$$ As such $$\sum_{G \in \mathcal{G}(y_i)}  \mathbb{E}_{\tau}\left[ \frac{\mathbf1\{G \in \tau\}}{C_i} \right] = d\mathbb{E}_{\tau}\left[ \frac{\mathbf1\{G \in \tau\}}{C_i} \right] = q_K, $$ and so, $$\mathbb{E}_{\tau}\left[ \frac{\mathbf1\{G \in \tau\}}{C_i} \right] = \frac{q_K}{d}.$$

This allows us to simplify: 
\begin{align*}
\mathbb E_{y_{1:N},\mathcal T} \left[ \sum_{i=1}^N \nabla_i \mathbf 1\{C_i>0\} \frac{1}{C_i} \sum_{\substack{S\in \mathcal T\\ i\in S}} f_S \right] &= \mathbb E_{y_{1:N}}
\left[
\sum_{i=1}^N
\nabla_i \sum_{G \in \mathcal{S}, y_i \in G} f(x, G) \mathbb{E}_{\tau}\left[ \frac{\mathbf 1\{G \in \tau\}}{C_i} \right] \right] \\
&= \frac{q_K}{d} \mathbb E_{y_{1:N}}
\left[
\sum_{i=1}^N
\nabla_i \sum_{G \in \mathcal{S}, y_i \in G} f(x, G) \right]  \\
&= \frac{q_K}{d} \mathbb E_{y_{1:N}}
\left[ \sum_{G \in \mathcal{S}} f(x, G)
\sum_{y_i \in G}
\nabla_i \right]  \\
&= \frac{q_K}{d} \sum_{G \in \mathcal{S}}  \mathbb E_{y_{1:N}}
\left[ f(x, G)
\sum_{y_i \in G}
\nabla_i \right]  \\
&= \frac{q_K}{d} \binom{N}{n} \nabla_\theta \mathbb{E}_{y_{1:n}}[f(x, y_{1:n})]. 
\end{align*}

Next, we simplify $\text{Term 2}$. Since $ \widehat f_{\mathcal T}(x) = \frac{1}{K}\sum_{U\in \mathcal T} f(x, U),$ we have
\begin{align*}
    \mathbb{E}_{y_{1:N}, \tau}\left[ \sum_{i=1}^N \nabla_i \mathbf{1}\{ c_i > 0\} \widehat f_{\mathcal T}(x)\right] &= \mathbb E_{y_{1:N}} \left[ \sum_{i=1}^N \nabla_i \sum_{U\in \mathcal S} f(x, U) \mathbb E_{\mathcal T} \left[ \mathbf 1\{C_i>0\} \frac{\mathbf 1\{U\in \mathcal T\}}{K} \right] \right].
\end{align*}

Note that the inner expectation is conditioned on a fixed $U$. Now, for a fixed $U$ and a fixed $i$, there are two cases. Case 1: $y_i \in U$. In this case $\mathbb E_{\mathcal T} \left[ \mathbf 1\{C_i>0\} \frac{\mathbf 1\{U\in \mathcal T\}}{K} \right] = \frac{1}{K}\mathbb{P}(U \in \mathcal{T})= \frac{1}{K}\cdot \frac{K}{K_\mathrm{all}} = \frac{1}{K_\mathrm{all}}$. Case 2: $y_i \not\in U$. In this case, $f(x, U)$ is independent of $y_i$, in which case we use the fact that $\mathbb{E}_{y_{1:N}}[f(x, U)\nabla_i] = 0$. So, we can simplify as: 

\begin{align*}
    \mathbb{E}_{y_{1:N}, \tau}\left[ \sum_{i=1}^N \nabla_i \mathbf{1}\{ c_i > 0\} \widehat f_{\mathcal T}(x)\right] &= \mathbb E_{y_{1:N}} \left[ \sum_{i=1}^N \nabla_i \sum_{U\in \mathcal S} f(x, U) \mathbb E_{\mathcal T} \left[ \mathbf 1\{C_i>0\} \frac{\mathbf 1\{U\in \mathcal T\}}{K} \right] \right] \\
    &= \mathbb E_{y_{1:N}} \left[ \sum_{U \in \mathcal{S}}f(x, U) \sum_{y_i \in U} \frac{1}{K_\mathrm{all}} \nabla_i \right] \\
    &= \frac{1}{K_\mathrm{all}}  \sum_{U \in \mathcal{S}}\mathbb E_{y_{1:N}} \left[f(x, U) \sum_{y_i \in U} \nabla_i \right] \\
    &= \nabla_\theta \mathbb{E}_{y_{1:n}}\left[ f(x, y_{1:n})\right]. 
\end{align*}

Finally, we verify the expression for $q_K$. The event $C_i=0$ means that none of the $K$ sampled sets contains $i$. There are $K_\mathrm{all}-\binom{N-1}{n-1}$ sets in $\mathcal S$ that do not contain $i$, so
$$\Pr(C_i=0)
=
\frac{\binom{K_\mathrm{all}-d}{K}}{\binom{K_\mathrm{all}}{K}}.$$
Therefore,
$$q_K
=
\Pr(C_i>0)
=
1-\frac{\binom{K_\mathrm{all}-d}{K}}{\binom{K_\mathrm{all}}{K}}.$$

Combining these two terms and doing a routine check shows that the scaling factor is positive as long as $N > n$ and $K > 1$. \end{proof}

Now, we prove Lemma \ref{lem: probability_mass_shift_in_setRL}.  

\textbf{Lemma 5.1 (Restated)} In set reinforcement learning, assuming the objective function $f$ is symmetric in its arguments and the learning rate is $\alpha$, the shift in probability mass on a fixed generation $y$ after one step gradient update from $\pi_\theta^k$ to $\pi_\theta^{k+1}$ can be written as:
\begin{align*}
    \label{eq: logit_shift_in_set_RL}
    \log \pi_\theta^{k+1}(y \mid x) - \log \pi_\theta^{k}(y \mid x) 
    &= \alpha n \pi_{\theta}^k(y \mid x) \left[ \mathbb{E}_{y_{2:n}\sim \pi_\theta^k(\cdot \mid x)}\left[f(x, y, y_{2:n}) \right] - \mathbb{E}_{y_{1:n}\sim \pi_\theta^k(\cdot \mid x)}\left[f(x,  y_{1:n}) \right]\right] - \alpha C(\theta^k).
\end{align*}

\begin{proof}
This follows from \cite{hamid2026polychromicobjectivesreinforcementlearning} by noting that when a set is homogeneous, then it's scaffold value, $\Lambda_{f}(y^{\oplus n}; \pi_\theta^k, x)$, can be written as the probability of sampling $y$ times the marginal set advantage of $y$ - both under $\pi_\theta^k$. 
\end{proof}

\end{document}

%% file: macros.tex
\usepackage[table,dvipsnames,xcdraw]{xcolor}
\usepackage{colortbl}

\definecolor{polychrome_blue}{HTML}{045A8D}

\usepackage[most]{tcolorbox}
\usepackage{siunitx} 
\usepackage{array}
\usepackage{tabularx}
\usepackage{sectsty}
\usepackage{graphicx} 
\usepackage[toc]{appendix}
\usepackage{upquote}     
\usepackage{listings}
\usepackage{comment}

\usepackage{amsmath}
\usepackage{amssymb}
\usepackage{amsfonts}
\usepackage{fancyvrb,fvextra}
\usepackage{multirow}
\usepackage{booktabs} 
\usepackage{soul}
\usepackage{alltt}

\definecolor{lightgreen}{rgb}{0.88, 1, 0.88}

\definecolor{softgreen}{RGB}{224, 245, 210}  

\definecolor{pastelyellow_full}{RGB}{250, 238, 135}
\colorlet{pastelyellow}{pastelyellow_full!70}

\definecolor{D}{HTML}{a0e7a0}          

\DefineVerbatimEnvironment{MyVerbatim}{Verbatim}{
  commandchars=@\{\},
  breaklines=true,
  breaksymbol={}, 
}

\DefineVerbatimEnvironment{alltt2}{Verbatim}{
  breaklines=true,
  breaksymbol={}, 
}

\newtcolorbox{mybox}[1][]{
    title=#1,
    fonttitle=\small,
    fontupper=\small,
    left=2mm,
    right=2mm,
    top=1mm,
    bottom=0mm,
}

\usepackage[
    backend=biber,
    style=alphabetic,
    natbib=true,
    giveninits=true,
    maxalphanames=3,
    minalphanames=3,
    maxbibnames=99
]{biblatex}

\DefineBibliographyStrings{english}{%
  urlseen = {},
  urlfrom = {},
}

\usepackage[hidelinks]{hyperref}
\hypersetup{
    colorlinks=true,
    linkcolor=polychrome_blue,
    urlcolor=polychrome_blue,
    citecolor=polychrome_blue,
    filecolor=magenta
}

\usepackage[capitalize]{cleveref}

\usepackage{listings}

\lstdefinestyle{promptstyle}{
    basicstyle=\ttfamily\small,
    breaklines=true,
    columns=fullflexible,
    frame=none
}

%% file: math_commands.tex
\usepackage{amsmath,amsfonts,bm}

\def\eqref#1{equation~\ref{#1}}

\def\1{\bm{1}}

\DeclareMathAlphabet{\mathsfit}{\encodingdefault}{\sfdefault}{m}{sl}
\SetMathAlphabet{\mathsfit}{bold}{\encodingdefault}{\sfdefault}{bx}{n}

\newtheorem{lem}{Lemma}[section]

\newtheorem{prop}[lem]{Proposition}

\newenvironment{proof}{\noindent\textit{Proof.}}{\hfill$\square$}

